\documentclass[sigconf, authorversion=true]{acmart}

\usepackage[T1]{fontenc}
\usepackage{booktabs} 
\usepackage{tabularx}
\usepackage{graphicx}
\usepackage{subcaption}

\usepackage{blindtext, graphicx}
\usepackage{algorithm,algorithmic}
\usepackage[algo2e]{algorithm2e} 
\renewcommand{\algorithmiccomment}[1]{\bgroup\hfill//~#1\egroup}
\usepackage[cmex10]{mathtools}
\usepackage{subcaption}
\usepackage{amssymb}
\usepackage{amsmath}
\usepackage{amsthm}

\usepackage{empheq}
\usepackage{xr}

\usepackage{paralist}

\newtheorem{problem}{Problem}

\newtheorem*{theorem*}{Theorem}

\AtBeginDocument{%
  \providecommand\BibTeX{{%
    \normalfont B\kern-0.5em{\scshape i\kern-0.25em b}\kern-0.8em\TeX}}}

\copyrightyear{2020}
\acmYear{2020}
\setcopyright{acmlicensed}\acmConference[KDD '20]{Proceedings of the 26th
ACM SIGKDD Conference on Knowledge Discovery and Data Mining}{August
23--27, 2020}{Virtual Event, CA, USA}
\acmBooktitle{Proceedings of the 26th ACM SIGKDD Conference on Knowledge
Discovery and Data Mining (KDD '20), August 23--27, 2020, Virtual Event, CA,
USA}
\acmPrice{15.00}
\acmDOI{10.1145/3394486.3403173}
\acmISBN{978-1-4503-7998-4/20/08}

\begin{document}
\fancyhead{}

\title{Block Model Guided Unsupervised Feature Selection}

\author{Zilong Bai}
\orcid{}
\affiliation{%
  \institution{University of California, Davis}
}
\email{zlbai@ucdavis.edu}

\author{Hoa Nguyen}
\orcid{}
\affiliation{%
  \institution{University of California, Davis}
}
\email{hoanguyen@ucdavis.edu}

\author{Ian Davidson}
\affiliation{%
  \institution{University of California, Davis}
}
\email{davidson@cs.ucdavis.edu}
\renewcommand{\shortauthors}{Z. Bai, et al.}

\begin{abstract}
Feature selection is a core area of data mining with a recent innovation of graph-driven unsupervised feature selection for linked data. In this setting we have a dataset $\mathbf{Y}$ consisting of $n$ instances each with $m$ features and a corresponding $n$ node graph (whose adjacency matrix is $\mathbf{A}$) with an edge indicating that the two instances are similar. Existing efforts for unsupervised feature selection on attributed networks have explored either directly regenerating the links by solving for $f$ such that $f(\mathbf{y}_i,\mathbf{y}_j) \approx \mathbf{A}_{i,j}$ or finding community structure in $\mathbf{A}$ and using the features in $\mathbf{Y}$ to predict these communities. However, graph-driven unsupervised feature selection remains an understudied area with respect to exploring more complex guidance. Here we take the novel approach of first building a block model on the graph and then using the block model for feature selection. That is, we discover $\mathbf{F}\mathbf{M}\mathbf{F}^T \approx \mathbf{A}$ and then find a subset of features $\mathcal{S}$ that induces another graph to preserve both $\mathbf{F}$ and $\mathbf{M}$. We call our approach Block Model Guided Unsupervised Feature Selection (BMGUFS). Experimental results show that our method outperforms the state of the art on several real-world public datasets in finding high-quality features for clustering. 
\end{abstract}


\maketitle

\section{Introduction}
The area of feature selection is a critical initial step in data mining and vital for its success. It has been extensively studied \cite{li2017feature} with a recent innovation of graph driven feature selection where in addition to an $m$ featured data set $\mathbf{Y}$ of $n$ instances, we are given an $n$ node graph whose adjacency matrix between instances is $\mathbf{A}$. Here the graph represents instance similarity such that if $\mathbf{A}_{a,b} \geq \mathbf{A}_{i,j}$ then instances  $a$ and $b$ are more similar than instances $i$ and $j$. This allows a rich source of guidance for the feature selection process.

Such a setting is not unusual in modern data mining particularly given the proliferation of attributed networks in various domains ranging from social media (e.g., Twitter \cite{taxidou2014online}) to biochemistry (e.g., protein-protein interacting networks \cite{safari2014protein}). In these settings the nodes are accompanied by a collection of features (an $n \times m$ feature matrix $\mathbf{Y}$) in addition to relational network topology (an $n \times n$ adjacency matrix $\mathbf{A}$). 
A challenge in these domains is that the nodal attributes can be a noisy/irrelevant or even redundant  high-dimensional feature space. This can yield suboptimal solutions if we assume all the features associated with the nodes and the graph structure are complementary \cite{sanchez2015efficient, zhe2019community}. 

Existing work to address this challenge takes two broad directions to make use of the graph. The first (micro-level) is to learn a function that maps the feature vectors of two instances to a value that approximates their edge weight in the graph, that is $f(\mathbf{y}_i,\mathbf{y}_j) \approx \mathbf{A}_{i,j}$ (e.g., \cite{wei2015efficient, wei2016unsupervised, li2019adaptive}). A second direction (macro-level) includes finding communities from $\mathbf{A}$ either explicitly (e.g. \cite{tang2012unsupervised}) or implicitly (e.g., \cite{li2016robust}) and selecting features to predict them.
Instead, we take the novel approach of finding a block model $\mathbf{F} \mathbf{M} \mathbf{F}^T \approx \mathbf{A}$ and use the block model $\mathbf{F} ~\text{and}~ \mathbf{M}$ to guide the feature selection.
This is different from existing work in two ways. Firstly, clustering and block modeling are not the same, as in block modeling two instances are placed in the same block if they are structurally equivalent (e.g., second-order proximity \cite{zhang2018network, tang2015line}), not if they belong to the same densely connected subgraph (i.e., intra-community proximity \cite{zhang2018network,girvan2002community}). Secondly, a block model effectively denoises the graph and hence removes noisy edges. See Figure \ref{fig:block_model_guided_feature_selection_diagram} for an illustration of our work. 

Our major contributions are:
\begin{compactenum}
\item We propose a novel block-model driven formulation for feature selection (section \ref{sec:formulation}).
\item We derive an effective numerical optimization framework for our formulation (section \ref{sec:solver}).
\item We empirically demonstrate the usefulness of our method and investigate its potential via extensive experiments on several real-world public datasets (section \ref{sec:experiments}). 
\begin{compactenum}
\item We demonstrate the effectiveness of our method in finding high-quality features to facilitate K-means clustering. Our method outperforms the baselines on various real-world public datasets (section \ref{sec:effectiveness}).
\item We conduct in-depth analysis on the sensitivity of our method w.r.t. the block models generated from the structural graph to gain insights for future endeavor beyond our explorations (section \ref{sec:block_model_effect}).
\end{compactenum}
\end{compactenum}

\begin{figure*}[ht!]
\centering
\includegraphics[width=5in]{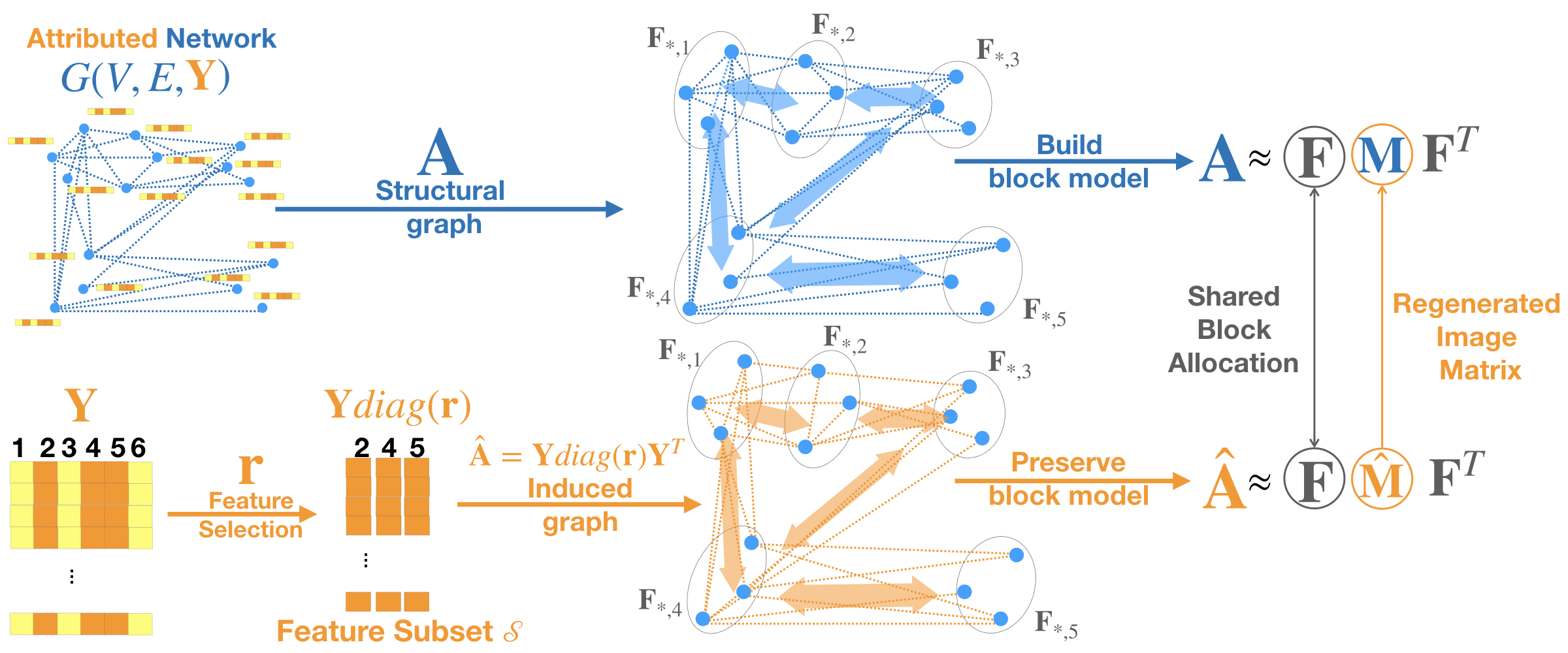}
\caption{Diagram of our proposed BMGUFS. In the above attributed network $G(V, E, \mathbf{Y})$, we first compute the block model $\mathbf{F}$ and $\mathbf{M}$ for the structural graph $\mathbf{A}$. We then select features $\mathcal{S}$ to induce graph $\hat{\mathbf{A}} = \mathbf{Y} diag(\mathbf{r}) \mathbf{Y}^T$ to maximally preserve $\mathbf{F}$ and $\mathbf{M}$. Particularly, features in orange (i.e., $2,4,5$) are selected such that on their induced $\hat{\mathbf{A}}$: (1) the given block allocation $\mathbf{F}$  minimally violates structural equivalence and (2) the image matrix $\hat{\mathbf{M}}$ corresponded to $\mathbf{F}$ maximally regenerates $\mathbf{M}$.} \label{fig:block_model_guided_feature_selection_diagram}
\end{figure*}

We begin the rest of the paper by presenting the problem setting in section \ref{sec:problem_setting}. We then formulate our BMGUFS as an optimization problem in section \ref{sec:formulation} and derive a highly effective algorithm in section \ref{sec:solver}. We present the results of our extensive experiments in section \ref{sec:experiments}. We then briefly review highly related work in section \ref{sec:related_work}. We conclude the paper and discuss future directions in section \ref{sec:conclusion}.

\section{Problem Setting}\label{sec:problem_setting}
In this section we first establish our notation in Table \ref{tab:notations}. We then present several concepts as preliminaries of our Block Model Guided Unsupervised Feature Selection problem. We formally define the novel feature selection problem we explore in Problem \ref{def:block_model_guided_feature_selection}.

\begin{table}[b]
\begin{tabular}{ |p{0.08\textwidth} | p{0.36\textwidth} |}
\hline 
\textbf{Notation} & \textbf{Definition}  \\
\hline 
$n$ & Number of nodes. \\
\hline 
$m$ & Number of original features. \\
\hline 
$d$ & Number of selected features. \\
\hline 
$k$ & Number of blocks in a block model. \\
\hline 
$\mathcal{Y}$ & Original feature set. \\
\hline 
$\mathcal{S}$ & Selected feature set. \\
\hline 
$\mathbf{Y}^{n \times m}$ & Feature matrix of all $n$ nodes. \\
\hline
$\mathbf{A}^{n \times n}$ & Adjacency matrix of the structural graph for the attributed network. \\
\hline
$\hat{\mathbf{A}}^{n \times n}$ & Adjacency matrix of the similarity graph induced from $\mathcal{S}$. \\
\hline
$\mathbf{F}^{n \times k}$ & Block allocation matrix with $k$-blocks stacked in columns. \\
\hline
$\mathbf{M}^{k \times k}$ & Image matrix on the structural graph. \\
\hline
$\hat{\mathbf{M}}^{k \times k}$ & Image matrix on the induced graph. \\
\hline
$\mathbf{r} \in \{0,1\}^m$ & The feature selection indicator vector. \\
\hline
$\mathbf{r}\in [0,1]^m$ & Importance scores for all the features.  \\
\hline
\end{tabular}\caption{Notations and Definitions} \label{tab:notations}  
\end{table}

We denote matrices as boldface capital letters (e.g., $\mathbf{X}$), vectors as boldface lowercase letters (e.g., $\mathbf{x}$), scalars as regular lowercase letters (e.g., $x$). We index the $i$-th entry of vector $\mathbf{x}$ with $x_i$, the $i$-th row of $\mathbf{X}$ with $\mathbf{X}_{i,*}$, the $j$-th column of this matrix with $\mathbf{X}_{*,j}$, the $(i,j)$ entry in $\mathbf{X}$ with $\mathbf{X}_{i,j}$. We use either $\mathbf{X}'$ or $\mathbf{X}^T$ to denote the transpose of $\mathbf{X}$. We use $tr(\mathbf{X})$ to denote the trace of square matrix $\mathbf{X}$. We follow \texttt{MATLAB} syntax to use $diag(\bullet)$ for \emph{either} diagonalization when $\bullet$ is a vector $\mathbf{x}$ \emph{or} extracting the $m$ diagonal entries as a vector when $\bullet$ is a square matrix $\mathbf{X}^{m \times m}$. We use $\mathbf{1}$ to denote a vector with all elements being $1$, $\bar{\mathbf{1}} = \frac{\mathbf{1}}{\| \mathbf{1} \|_2}$. For matrix/vector computations, we use $\otimes$ for Kronecker product, $\odot$ for Hadamard (element-wise) product, and $\frac{\bullet}{\bullet}$ for element-wise division. Horizontal concatenation of two matrices refers to regular matrix product. We use $\mathbf{X}(\mathbf{r})$ to denote a matrix $\mathbf{X}$ is a matrix function of $\mathbf{r}$. Function $nnz(\mathbf{r})$ counts the number of non-zero entries in $\mathbf{r}$.

\begin{definition}[\textbf{Attributed Network}]\label{def:attributed_network}
An attributed network $G(V, E, \mathbf{Y})$ consists of the set of $n$ nodes $V$, the set of links $E \subset V \times V$, and $\mathbf{Y}^{n \times m}$ where $\mathbf{Y}_{i,*}$ is the $m$-dimensional feature/attribute vector of node $v_i$. The adjacency matrix of the \textbf{Structural Graph} (i.e., the raw network topology) is $\mathbf{A}$.
\end{definition}

\begin{definition}[\textbf{Graph Induced by $\mathcal{S}$}]\label{def:behavioral_graph}  
Let $\mathcal{S}$ be the subset of $d$ features selected from the original $m$-dimensional feature space $\mathcal{Y}$. The graph induced by $\mathcal{S}$ is defined by the similarity between nodes in $\mathcal{S}$.
Formally, let $\mathbf{r} \in \{0,1\}^m$ indicating the $l-$th feature is selected \texttt{iff} $r_l = 1$, otherwise $0$, the adjacency matrix of the graph induced by $\mathcal{S}$ is defined as $\hat{\mathbf{A}} = \mathbf{Y} diag(\mathbf{r}) \mathbf{Y}^T$ for this paper.
\end{definition}

\begin{definition}[\textbf{Block Model of Graph $\mathbf{A}$}]\label{def:block_model_result}
A block modeling result of graph $\mathbf{A}^{n \times n}$ that partitions its node set $V$ into $k$ blocks consists of a block allocation matrix $\mathbf{F} \in \{0,1\}^{n \times k}$ and an image/mixing matrix $\mathbf{M} \in R_+^{k \times k}$, s.t. $\mathbf{F}, \mathbf{M}$ (approximately) minimize $\| \mathbf{A} - \mathbf{F} \mathbf{M} \mathbf{F}^T \|_F$. The image matrix $\mathbf{M}$ corresponded to $\mathbf{F}$ on graph $\mathbf{A}$ is $\underset{\mathbf{X}}{argmin} \| \mathbf{A} - \mathbf{F} \mathbf{X} \mathbf{F}^T \|_F$.
\end{definition}

\begin{problem}[\textbf{Block Model Guided Unsupervised Feature Selection}]\label{def:block_model_guided_feature_selection}

\noindent \textbf{Input:} Feature matrix $\mathbf{Y}$ for $n$ nodes in the original feature space $\mathcal{Y}$ with $m$ features, block model $\mathbf{F}$ and $\mathbf{M}$ precomputed from the adjacency matrix $\mathbf{A}$ of the structural graph over the $n$ nodes. \\
\noindent \textbf{Optimization:}  Find a subset of $d$ features $\mathcal{S}$ from $\mathcal{Y}$ ($d << m$), such that the graph induced from $\mathcal{S}$ (i.e., $\hat{\mathbf{A}}$) maximally preserves $\mathbf{F}$ and $\mathbf{M}$. \\
\noindent \textbf{Output:} An $m$-dimensional feature selection indicator vector $\mathbf{r} \in \{0,1\}^m$ where $r_l = 1$ \texttt{iff} feature $l$ from $\mathcal{Y}$ is in $\mathcal{S}$, $0$ otherwise.
\end{problem}

\section{Formulation}\label{sec:formulation}
In this section we formulate the Block Model Guided Unsupervised Feature Selection as an optimization problem. We aim to find a subset of features $\mathcal{S}$, such that a given block model $\mathbf{F}, \mathbf{M}$ precomputed for the structural graph $\mathbf{A}$ is maximally preserved on the graph $\hat{\mathbf{A}}$ induced by $\mathcal{S}$. This consists of two objectives: (1) block allocation $\mathbf{F}$ minimally violates structural equivalence on $\hat{\mathbf{A}}$ and (2) the image matrix $\hat{\mathbf{M}}$ corresponded to $\mathbf{F}$ on $\hat{\mathbf{A}}$ regenerates the given $\mathbf{M}$ up to scaling. We model the two objectives as $\mathcal{L}_b$ (section \ref{sec:preserve_F}) and $\mathcal{L}_m$ (section \ref{sec:regenerate_M}) respectively. 

Based on theorem \ref{theorem:closed_form_M}, $\hat{\mathbf{M}} (\mathbf{r})$ is a matrix function of feature selection vector $\mathbf{r}$ given $\mathbf{Y}$ and $\mathbf{F}$ (equation \ref{eq:Mhat_function_of_r}). Therefore, both $\mathcal{L}_b$ and $\mathcal{L}_m$ are functions of $\mathbf{r}$ without involving an independent variable matrix to model the image matrix $\hat{\mathbf{M}}$ corresponded to $\mathbf{F}$ on the induced graph.

\begin{equation}\label{eq:Mhat_function_of_r}
\begin{aligned}
\hat{\mathbf{M}}(\mathbf{r}) & = \underset{\mathbf{X}}{argmin} \| \mathbf{Y} diag(\mathbf{r}) \mathbf{Y}^T - \mathbf{F} \mathbf{X} \mathbf{F}^T \|_F \\
& = ( \mathbf{F}^T \mathbf{F})^{-1}\mathbf{F}^T \mathbf{Y} diag(\mathbf{r}) \mathbf{Y}^T \mathbf{F} ( \mathbf{F}^T \mathbf{F})^{-1} \\
\end{aligned}
\end{equation}

\begin{theorem}[Least Squares Optimal $\mathbf{M}$ in Closed Form] \label{theorem:closed_form_M}
Given $\mathbf{A} \in R^{n \times n}$, $\mathbf{F} \in \{0,1\}^{n\times k}$.  If $\mathbf{D} = \mathbf{F}^T \mathbf{F}$ is a diagonal matrix with positive diagonal elements, then 
\begin{equation}\label{eq:M_close_form_least_squares}
\underset{\mathbf{X}}{argmin} \| \mathbf{A} - \mathbf{F} \mathbf{X} \mathbf{F}^T \|^2_F = \mathbf{D}^{-1}\mathbf{F}^T \mathbf{A} \mathbf{F} \mathbf{D}^{-1}
\end{equation}
\end{theorem}
\begin{proof}
See Appendix \ref{sec:proof}.
\end{proof}

\subsection{Preserving Structural Equivalence with $\mathbf{F}$}\label{sec:preserve_F}
Here we aim to find a feature subset $\mathbf{S}$ such that on its induced graph $\hat{\mathbf{A}}$, block allocation $\mathbf{F}$ minimally violates the structural equivalence. According to \cite{mattenet2019generic}, the reconstruction error $\| \hat{\mathbf{A}} - \mathbf{F} \hat{\mathbf{M}} \mathbf{F}^T \|$ quantifies the violation of structural equivalence in using $\mathbf{F}$ and $\hat{\mathbf{M}}$ to model $\hat{\mathbf{A}}$. Note that the scale of absolute reconstruction error favors fewer entries in $\mathbf{r}$ to be positive, which can yield trivial solutions (e.g., $\mathbf{r} = \mathbf{0}$) instead of exploring more meaningful block models. Therefore, we model the loss term in equation \ref{eq:block_model_behavioral_loss} with the relative reconstruction error for $\hat{\mathbf{A}} = \mathbf{Y} diag(\mathbf{r}) \mathbf{Y}^T$. In equation \ref{eq:block_model_behavioral_loss}, $\mathbf{D} = \mathbf{F}^T \mathbf{F}$.

\begin{equation}\label{eq:block_model_behavioral_loss}
\begin{aligned}
\mathcal{L}_b(\mathbf{r}) = & \frac{\| \mathbf{Y}diag(\mathbf{r})\mathbf{Y}^T - \mathbf{F} ~ \hat{\mathbf{M}}(\mathbf{r}) ~ \mathbf{F}^T \|^2_F}{\| \mathbf{Y}diag(\mathbf{r})\mathbf{Y}^T \|^2_F}\\
\overset{(Eq. \ref{eq:Mhat_function_of_r})}{=} & \frac{\| \mathbf{Y} diag(\mathbf{r}) \mathbf{Y}^T - \mathbf{F} \mathbf{D}^{-1} \mathbf{F}^T \mathbf{Y} diag(\mathbf{r}) \mathbf{Y}^T \mathbf{F} \mathbf{D}^{-1} \mathbf{F}^T \|^2_F}{\| \mathbf{Y} diag(\mathbf{r}) \mathbf{Y}^T \|^2_F} 
\end{aligned}
\end{equation}

\subsection{Regenerating Image Matrix $\mathbf{M}$}\label{sec:regenerate_M}
Here we aim to find a feature subset $\mathcal{S}$ such that the image matrix $\hat{\mathbf{M}}$ corresponded to $\mathbf{F}$ on the graph $\hat{\mathbf{A}}$ induced by $\mathcal{S}$ (approximately) regenerates the given $\mathbf{M}$.
The underlying premise is that given the same block allocation $\mathbf{F}$, we want the block-level similarity (i.e., $\hat{\mathbf{M}}(\mathbf{r})$) on $\hat{\mathbf{A}}$ to respect the block-level connectivity (i.e., $\mathbf{M}$) on the structural graph $\mathbf{A}$. 
This translates to $C \odot \hat{\mathbf{M}} \approx \mathbf{M}$ where $C$ is a scalar that compensates for the scaling difference between the two image matrices. 
It is challenging to directly model and solve for $C$ as it is not only unknown but also dynamic as the scale of $\hat{\mathbf{M}}(\mathbf{r})$ changes with $\mathbf{r}$. 
Therefore, we define distance between $\hat{\mathbf{M}}(\mathbf{r})$ and $\mathbf{M}$ in equation \ref{eq:matching_image_matrices} invariant to scaling. 
We use $\mathcal{D}_{KL}(\bullet || \bullet)$ to denote the KL-divergence \cite{kullback1997information} between two discrete probabilistic distributions. In this paper we only consider block modeling result whose image matrix $\mathbf{M}$ \emph{does not} contain absolutely zero elements as the presence of absolute zero entries in $\mathbf{M}$ can pose additional challenges to understanding the stochastic properties of the block model \cite{abbe2017community}. 
\begin{equation}\label{eq:matching_image_matrices}
\begin{aligned}
\mathcal{L}_m(\mathbf{r}) & = \underset{\forall i \in [k]}{\Sigma} \mathcal{D}_{KL}(~ \mathbf{q}_i(\mathbf{r}) ~ ||~ \mathbf{p}_i ~) \\
\text{where} ~ &  ~ \mathbf{p}_i = \mathbf{P}_{i,*}, ~ \mathbf{q}(\mathbf{r})_i = \mathbf{Q}(\mathbf{r})_{i,*},  \\
 ~ & ~ \mathbf{P}_{i,j} = \frac{\mathbf{M}_{i,j}}{\underset{j' \in [k]}{\Sigma}~\mathbf{M}_{i,j'}}, ~ \mathbf{Q}(\mathbf{r})_{i,j} = \frac{\hat{\mathbf{M}}(\mathbf{r})_{i,j}}{\underset{j' \in [k]}{\Sigma}~\hat{\mathbf{M}}(\mathbf{r})_{i,j'}} \\
\end{aligned}
\end{equation}

\noindent \textbf{A Statistical Interpretation.} In the given $\mathbf{M}$, $\mathbf{M}_{i,j} \in [0,1]$ can be interpreted as the empirical probability of having an edge connecting two nodes between blocks $\mathbf{F}_{*,i}$ and $\mathbf{F}_{*,j}$. This induces the conditional probability given $\mathbf{F}_{*,i}$ to connect with $\mathbf{F}_{*,j}$ as $\mathbf{P}_{i,j} = \frac{\mathbf{M}_{i,j}}{\underset{j' \in [k]}{\Sigma} \mathbf{M}_{i,j'}}$. Assuming $\hat{\mathbf{M}}(\mathbf{r}) \propto \mathbf{M}$, we can define $\mathbf{Q}(\mathbf{r})_{i,j} = \frac{\hat{\mathbf{M}}(\mathbf{r})_{i,j}}{\underset{j' \in [k]}{\Sigma} \hat{\mathbf{M}}(\mathbf{r})_{i,j'}}$ to model the conditional probability given $\mathbf{F}_{*,i}$ to connect to $\mathbf{F}_{*,j}$ on the induced graph. Thus equation \ref{eq:matching_image_matrices} models the overall KL-divergence between the conditional probabilities of connectivity on the original graph $\mathbf{A}$ and the induced graph $\hat{\mathbf{A}}(\mathbf{r})$ at the block-level.

\subsection{A Joint Formulation}
We aim to holistically utilize both the block allocation $\mathbf{F}$ and the image matrix $\mathbf{M}$ to regularize the macro-level structure of the graph induced by the selected features.
Therefore, we combine $\mathcal{L}_b$ and $\mathcal{L}_m$ into a unified optimization framework in equation \ref{eq:formulation_binary_r} with an adaptive weighting factor $\dot{\beta} \geq 0$ \footnote{This is not the hyper-parameter $\bar{\beta} \in [0,1]$ for our algorithm \ref{alg:bmgufs}.}.
\begin{equation}\label{eq:formulation_binary_r} 
\begin{aligned}
\underset{\mathbf{r}}{Minimize}  & ~ \mathcal{L} = \mathcal{L}_b + \dot{\beta} \mathcal{L}_m \\
s.t. & ~ \mathbf{r} \in \{0,1\}^m, ~ \mathbf{r}^T \mathbf{1} = d
\end{aligned}
\end{equation}

To side step the potential intractability caused by combinatorial optimization, we relax the domain of $\mathbf{r}$ from  $\{0,1\}^m$ to $[0,1]^m$. The resulting $\mathbf{r}$ can be interpreted as importance scores for ranking the features. We then follow the convention of \cite{li2019adaptive} to rewrite the cardinality constraint $\mathbf{r}^T \mathbf{1} = d$ in the Lagrangian, resulting in the following constrained optimization problem with $l$-1 norm regularization (where $\gamma$ denotes the weight for sparsity penalty). We further notice that both $\mathcal{L}_b$ and $\mathcal{L}_m$ are invariant to the $l$-2 norm of $\mathbf{r}$. Therefore, we introduce $l$-2 norm constraint $\|\mathbf{r}\|_2=1$ to confine the search domain for our gradient-descent based algorithm. Equation \ref{eq:formulation_norm_r} presents the resulting relaxed formulation.

\begin{equation}\label{eq:formulation_norm_r}
\begin{aligned}
\underset{\mathbf{r}}{Minimize}  & ~ \mathcal{L} = \mathcal{L}_b + \dot{\beta} \mathcal{L}_m + \gamma \| \mathbf{r} \|_1\\
s.t. & ~ \mathbf{r} \geq \mathbf{0}, ~~ \|\mathbf{r}\|_2 = 1
\end{aligned}
\end{equation}

\section{Solver}\label{sec:solver}
In this section we derive an effective solver for equation \ref{eq:formulation_norm_r} to find a feature selection vector $\mathbf{r}$ given block model $\mathbf{F}, \mathbf{M}$.
Firstly, we compute the partial derivatives of $\mathcal{L}_b$ and $\mathcal{L}_m$ w.r.t. $\mathbf{r}$. We then suggest an update rule for $\mathbf{r}$ based on a weighted combination of the \emph{normalized} gradients. We summarize our optimization framework in algorithm \ref{alg:bmgufs}.

The derivation of $\frac{\partial \mathcal{L}_b}{\partial \mathbf{r}}$ is relatively straightforward - we induce equation \ref{eq:gradient_L_b} from equation \ref{eq:pLb_pr}.

\begin{equation}\label{eq:pLb_pr}
\begin{aligned}
\frac{\partial \mathcal{L}_b}{\partial \mathbf{r}} = & \frac{1}{\| \mathbf{Y} ~ \mathbf{R} ~ \mathbf{Y}' \|^4_F}(\| \mathbf{Y} ~ \mathbf{R} ~ \mathbf{Y}' \|^2_F \frac{\partial \| \mathbf{Y}diag(\mathbf{r})\mathbf{Y}^T - \mathbf{F} ~ \hat{\mathbf{M}}(\mathbf{r}) ~ \mathbf{F}^T \|^2_F}{\partial \mathbf{r}} \\
& -  \| \mathbf{Y}diag(\mathbf{r})\mathbf{Y}^T - \mathbf{F} ~ \hat{\mathbf{M}}(\mathbf{r}) ~ \mathbf{F}^T \|^2_F  \frac{\partial \| \mathbf{Y} ~ \mathbf{R} ~ \mathbf{Y}' \|^2_F}{\partial \mathbf{r}})
\end{aligned}
\end{equation}

\smallskip
\noindent\fbox{\begin{minipage}{26em}
\noindent \textbf{Gradient of $\mathcal{L}_b$ over $\mathbf{r}$}

\begin{equation}\label{eq:gradient_L_b}
\begin{aligned}
\frac{\partial \mathcal{L}_b}{\partial \mathbf{r}} ~ 
& = \frac{2 diag( \mathbf{Y}' ~ \mathbf{Y} ~ \mathbf{R} ~ \mathbf{Y}' ~ \mathbf{Y}  + \hat{\mathbf{Y}}' ~ \hat{\mathbf{Y}} ~ \mathbf{R} ~ \hat{\mathbf{Y}}' ~ \hat{\mathbf{Y}} - 2 \hat{\mathbf{Y}}' ~ \mathbf{Y} ~ \mathbf{R} ~ \mathbf{Y}' ~ \hat{\mathbf{Y}})}{\| \mathbf{Y} ~ \mathbf{R} ~ \mathbf{Y}' \|^2_F}  \\
& -  \frac{ 2 \mathcal{L}_b ~ diag(\mathbf{Y}' \mathbf{Y} ~ \mathbf{R} ~ \mathbf{Y}' ~ \mathbf{Y})}{\| \mathbf{Y} ~ \mathbf{R} ~ \mathbf{Y}' \|^2_F} \\
\text{where} ~ & ~  \mathbf{R} ~ = diag(\mathbf{r}), ~ \hat{\mathbf{Y}}  = \mathbf{F} \mathbf{D}^{-1} \mathbf{F}' \mathbf{Y}, ~ \mathbf{D} = \mathbf{F}^T \mathbf{F}
\end{aligned}
\end{equation}
\end{minipage}}

To derive $\frac{\partial \mathcal{L}_m}{\partial \mathbf{r}}$ we first compute $\frac{\partial \mathcal{L}_m}{\partial \mathbf{Q}}$ (equation \ref{eq:pLm_pQ}) and $\frac{\partial \mathbf{Q}_{i,j}}{\partial r_l}$ (equation \ref{eq:pQpr_element_form}).

\begin{equation}\label{eq:pLm_pQ}
\frac{\partial \mathcal{L}_m}{\partial \mathbf{Q}} = log \frac{\mathbf{Q}}{\mathbf{P}} + \mathbf{P}
\end{equation}

\begin{equation}\label{eq:pQpr_element_form}
\begin{aligned}
\frac{\partial \mathbf{Q}_{i,j}}{\partial r_l} & = \frac{1}{\underset{j'}{\Sigma} \hat{\mathbf{M}}_{i,j'}} [\bar{\mathbf{D}}_{i,l} \bar{\mathbf{D}}_{j,l}^T ~ - ~ \mathbf{Q}_{i,j} \bar{\mathbf{D}}_{i,l} \underset{j'}{\Sigma} \bar{\mathbf{D}}_{j',l}^T] \\
\text{where} ~ & \hat{\mathbf{M}} = \bar{\mathbf{D}} diag(\mathbf{r}) \bar{\mathbf{D}}^T, ~ \bar{\mathbf{D}} = \mathbf{D}^{-1} \mathbf{F}^T \mathbf{Y}
\end{aligned}
\end{equation}

Thus, we have the gradient of $\mathcal{L}_m$ over $r_l, \forall l \in [m]$ based on chain rule of partial derivations \footnote{In practice we add an ignorable positive scalar $\delta = 10^{-6}$ to $\mathbf{M}_{i,j}, \hat{\mathbf{M}}_{i,j}, \forall i,j \in [k]$ in computing $\mathbf{P}, \mathbf{Q}$ and $\frac{\partial \mathbf{Q}}{\partial r_l}, \forall l \in [m]$ to avoid numerical instability. The experimental results of this paper are indifferent to $\delta$ being $0$ or $10^{-6}$. }: 

\smallskip
\noindent\fbox{\begin{minipage}{26em}
\noindent \textbf{Gradient of $\mathcal{L}_m$ over $r_l$}

\begin{equation}\label{eq:gradient_L_m}
\begin{aligned}
\frac{\partial \mathcal{L}_m}{\partial r_l} & ~ = tr( ~ [log \frac{\mathbf{Q}}{\mathbf{P}} + \mathbf{P}]^T \frac{\partial \mathbf{Q}}{\partial r_l} ~ ) \\
\frac{\partial \mathbf{Q}}{\partial r_l} & = diag(\frac{1}{\hat{\mathbf{M}} \mathbf{1}})  [ \bar{\mathbf{D}}_{*,l} \bar{\mathbf{D}}_{*,l}^T - (\bar{\mathbf{D}}_{*,l}^T \mathbf{1}) diag(\bar{\mathbf{D}}_{*,l}) \mathbf{Q}  ~ ] \\
\end{aligned}
\end{equation}
where $\mathbf{Q}$ and $\mathbf{P}$ are computed according to equation \ref{eq:matching_image_matrices}.
\end{minipage}}

\subsection{Combining the Two Gradients to Update $\mathbf{r}$}
In this section, we compute the gradient w.r.t $\mathbf{r}$ to simultaneously optimize $\mathcal{L}_b$ and $\mathcal{L}_m$. It is conventional to combine the two gradients as $\frac{\partial \mathcal{L}_b}{\partial \mathbf{r}} +\beta \frac{\partial \mathcal{L}_m}{\partial \mathbf{r}}$ with a constant hyper-parameter $\beta$. However, we observe that objective $\mathcal{L}_b$ can dominate the minimization of $\mathcal{L}_b + \beta \mathcal{L}_m$. This can lead to \emph{increased} $\mathcal{L}_m$ unless $\beta$ is extremely large. According to our empirical study, the increment of $\mathcal{L}_m$ affects the quality of selected features, and it is difficult to search for a proper $\beta \in R_+$.  We alleviate this issue with a heuristic that combines the \emph{normalized} gradients proportionally according to a user-specified \textbf{composition ratio} $\bar{\beta} \in [0,1]$. With more confined hyper-parameter search space, this strategy is simple yet highly effective in practice to control the optimization of $\mathcal{L}_b$ and $\mathcal{L}_m$. 

Equation \ref{eq:joint_gradient} computes the combined gradient with sparsity penalty weight $\gamma$. We use Projected Gradient Descent (PGD) followed by rescaling/normalization to update $\mathbf{r}$, such that $\mathbf{r}$ satisfies both the non-negativity and $l$-2 norm constraints per iteration. The updating is formally defined as equation \ref{eq:update_r} performed in order, where $\eta^{(t)}$ is the step size at the $t$-th iteration.
\begin{equation}\label{eq:joint_gradient}
\frac{\partial \mathcal{L}}{\partial \mathbf{r}} =\frac{ (1 - \bar{\beta})}{\| \frac{\partial \mathcal{L}_b}{\partial \mathbf{r}} \|_2} \frac{\partial \mathcal{L}_b}{\partial \mathbf{r}} +\frac{ \bar{\beta}}{\| \frac{\partial \mathcal{L}_m}{\partial \mathbf{r}} \|_2} \frac{\partial \mathcal{L}_m}{\partial \mathbf{r}} + \gamma \bar{\mathbf{1}}
\end{equation}

\begin{equation}\label{eq:update_r}
\begin{aligned}
 & \mathbf{r} = \mathbf{r}^{(t)} - \eta^{(t)} \frac{\partial \mathcal{L}}{\partial \mathbf{r}} ; \\
~ &
r_l \leftarrow Max(r_l,0), ~ \forall l \in [m] ~ ; \\
~
& \mathbf{r}^{(t+1)} \leftarrow \frac{\mathbf{r}}{\| \mathbf{r} \|_2} \\
\end{aligned}
\end{equation}

We summarize our optimization framework for equation \ref{eq:formulation_norm_r} in algorithm \ref{alg:bmgufs}. Given $d$, a specific number of selected features, we select the top $d$ features with the largest importance scores in $\mathbf{r}\in[0,1]^m$.
We empirically demonstrate the convergence of $\mathcal{L}_b + \mathcal{L}_m$ with properly set $\bar{\beta}$ in section \ref{sec:solver_inspection}. Interestingly, we observe that our method can select high-quality features when $\bar{\beta}$ effectively reduces $\mathcal{L}_m$ (Figure \ref{fig:lm_vs_iterations}). We leave theoretical investigation on using equations \ref{eq:joint_gradient} and \ref{eq:update_r} as general purpose optimization technique to future endeavors.

\begin{algorithm}
\caption{Block Model Guided Unsupervised Feature Selection}\label{alg:bmgufs}
\begin{algorithmic}[1] 
\REQUIRE Block allocation $\mathbf{F}$, image matrix $\mathbf{M}$ (precomputed from structural graph $\mathbf{A}$), feature matrix $\mathbf{Y}$, ratio for combining gradients $\bar{\beta}$, sparsity regularization weight $\gamma$, maximum iterations $maxI$, (constant) step size $\eta$.
\STATE Initialize feature selection vector $\mathbf{r} = \frac{\mathbf{1}}{\| \mathbf{1} \|_2}$.
\WHILE{Termination Condition Unsatisfied}
\STATE Compute $\frac{\partial \mathcal{L}_b}{\partial \mathbf{r}}$ with equation \ref{eq:gradient_L_b} and $\frac{\partial \mathcal{L}_m}{\partial \mathbf{r}}$ with equation \ref{eq:gradient_L_m}.
\STATE Compute gradient $\frac{\partial \mathcal{L}}{\partial \mathbf{r}}$ with equation \ref{eq:joint_gradient} given $\bar{\beta}$ and $\gamma$.
\STATE Use $\frac{\partial \mathcal{L}}{\partial \mathbf{r}}$ and $\eta$ to update $\mathbf{r}$ based on PGD and rescaling to satisfy constraints with equation \ref{eq:update_r}.
\ENDWHILE
\RETURN: Feature selection vector $\mathbf{r}$.
\end{algorithmic}
\end{algorithm}

\noindent \textbf{Computational Complexity Analysis.} The computational cost of our algorithm for computing the gradients in each iteration is given by $\mathcal{O}(k^3 m + m^3)$ where $m$ is the original number of features, $k$ is the number of blocks in the block model (a very small integer). The number of nodes is irrelevant to the computational cost in each iteration if we precompute constant matrices to avoid redundant computations.

\section{Experiments}\label{sec:experiments}\footnote{Source codes available in \url{https://github.com/ZilongBai/KDD2020BMGUFS} for reproducibility.}
In this section, we extensively evaluate our method on various real-world public datasets to address the following questions: 

\begin{itemize}
\item{\textbf{Q1. Effectiveness of our method (section \ref{sec:effectiveness}).}} Can our method find high-quality features to facilitate downstream application (see Figure \ref{fig:clustering_performance})? 

\item{\textbf{Q2. Sensitivity to Block Model Guidance (section \ref{sec:block_model_effect}).} } The question is multi-facet and we focus on the following two in this paper due to space limitations:
\begin{compactenum}
\item Is our method sensitive to perturbations in block model guidance as input to our algorithm \ref{alg:bmgufs} (see Figure \ref{fig:bm_perturbation_sensitivity})?
\item Can \textbf{different} block models generated from the \textbf{same} structural graph offer different guidance (see Figure \ref{fig:block_model_sensitivity})?
\end{compactenum}

\item{\textbf{Q3. Sensitivity to Parameter Selection (section \ref{sec:hyperparameter_effect}).}} How do the composition ratio $\bar{\beta}$ and sparsity penalty $\gamma$ influence the clustering performance of the features selected by our method (see Figure \ref{fig:parameter_sensitivity})?

\item{\textbf{Q4. Solver Inspection (section \ref{sec:solver_inspection}).}} How does the composition ratio $\bar{\beta}$ influence the optimization process of the objective function of our model (see Figure \ref{fig:combined_objective_vs_iterations})?
\end{itemize}

\subsection{Experimental Settings}
\noindent \textbf{Datasets.} We test our method on three real-world public datasets: BlogCatalog \cite{huang2018exploring}, Citeseer (sparse graph) \cite{sen2008collective, kipf2016semi} and  Cora \cite{sen2008collective, kipf2016semi}.  
Table \ref{tab:stats_of_datasets} summarizes basic statistics of the three datasets. See Appendix \ref{sec:dataset_preprocessing} for details on dataset preprocessing.

\noindent \textbf{Baselines.} We compare with the following baselines to demonstrate the effectiveness of our method (\textbf{Q1}). We use the source codes provided by the paper authors to reproduce MMPOP and NetFS. We apply methods in \texttt{scikit-feature}\cite{li2017feature} to obtain the results of LapScore, SPEC, and NDFS. See Appendix \ref{sec:baseline_codes_hp} for links to their source codes and settings of their hyper-parameters. 
\begin{itemize}
\item{\textbf{All features}.}
\item{\textbf{LapScore}} \cite{he2006laplacian} evaluates the importance of a feature based on its power of preserving locality.
\item{\textbf{SPEC}} \cite{zhao2007spectral} proposes a unified framework for feature selection based on spectral graph theory.
\item{\textbf{NDFS}} \cite{li2012unsupervised} jointly learns cluster labels via spectral clustering \emph{and} feature selection matrix with $l_{2,1}$-norm regularization.
\item{\textbf{MMPOP}}\cite{wei2015efficient} selects features to maximally preserve local partial order on the structural graph. 
\item{\textbf{NetFS}} \cite{li2016robust} \footnote{According to the empirical evaluation in \cite{li2019adaptive}: (1) NetFS \cite{li2016robust} can achieve state-of-the-art ACC and NMI on BlogCatalog at $d = 200$ (even better than itself at $d \in \{600, 1000\}$). (2) ADAPT \cite{li2019adaptive} and NetFS can achieve similar performance - better than their baseline methods - on various datasets w.r.t. varying number of selected features.} embeds latent representation learning that respects network clustering into feature selection.  
\end{itemize}

\noindent \textbf{Metrics for Performance Evaluation.} We follow the convention \cite{li2017feature, li2019adaptive}  to use K-means clustering on selected features (after normalization) 
as downstream application to evaluate the quality of selected features. We follow the typical settings in \cite{li2017feature, yang2011l2} to use Accuracy (ACC in equation \ref{eq:acc}) and Normalized Mutual Information (NMI in equation \ref{eq:nmi}) as performance metrics. See Appendix \ref{appendix:metrics} for their detailed definitions.  Conventionally, the higher ACC and NMI, the higher quality of the features. We report the mean result after $20$ runs of K-means to compensate for randomness.

\begin{table}
\begin{tabular}{ | c | c | c | c |}
\hline 
\textbf{Statistic} & \textbf{BlogCatalog} &  \textbf{Citeseer} & \textbf{Cora}  \\
\hline 
\textbf{Nodes \#} & 5196   & 3312 & 2708 \\
 \hline
\textbf{Links \#} & 171743  & 4660 & 5278 \\
 \hline
\textbf{Features \#} & 8189  & 3703  & 1433 \\
 \hline
\textbf{Classes \#} & 6 & 6 & 7 \\
 \hline
\end{tabular} \caption{Summary on Statistics of Datasets}\label{tab:stats_of_datasets}  
\end{table}

\subsection{Building Block Models for Structural Graph}
As we discuss in the related work (section \ref{sec:related_work}), there exist a plethora of approaches for block modeling. We use the multiplicative update rules for the Orthogonal Nonnegative Matrix tri-Factorization (ONMtF) formulation (equation \ref{eq:ONMtF_block_model}) proposed by seminal work \cite{ding2006orthogonal} to generate multiple candidate block models. Since $\mathbf{F}, \mathbf{M}$ are not jointly convex in the formulation, we can harvest multiple (i.e., $10$) different block models based on random initializations for each dataset. Each block model is computed with $100$ iterations\footnote{Multiplicative update rules are recognized to converge slowly in solving NMF formulations \cite{lin2007projected}. We set the maximum iterations to $100$ where the objective function does not observably decrease.}. We convert $\mathbf{F} \in [0,1]^{n \times k}$ to $\mathbf{F} \in \{0,1\}^{n \times k}$ by setting the largest entry on each row to $1$, others to $0$. The number of blocks $k$ is set to the number of classes for each dataset. We then compute $\mathbf{M}$ based on equation \ref{eq:M_close_form_least_squares} in theorem \ref{theorem:closed_form_M}. The $10$ different block models are identified by $\# i, i \in [10]$ according to the order they were generated.
\begin{equation}\label{eq:ONMtF_block_model}
\underset{\mathbf{F} \geq \mathbf{0}, \mathbf{M} \geq \mathbf{0}}{Minimize} \| \mathbf{A} - \mathbf{F} \mathbf{M} \mathbf{F}^T \|_F ~ s.t., ~ \mathbf{F}^T \mathbf{F} = \mathbf{I}
\end{equation}

We define relative reconstruction error (RRE) of using block model $\mathbf{F}, \mathbf{M}$ for reconstructing adjacency matrix $\mathbf{A}$ in equation \ref{eq:rre} to facilitate block model selection \emph{before} running our algorithm \ref{alg:bmgufs}. 
\begin{equation}\label{eq:rre}
RRE(\mathbf{F}, \mathbf{M}, \mathbf{A}) = \frac{\| \mathbf{A} - \mathbf{F} \mathbf{M} \mathbf{F}^T \|_F}{ \| \mathbf{A} \|_F }
\end{equation}

\subsection{Effectiveness of Our Method}\label{sec:effectiveness}
We demonstrate the effectiveness of our method by comparing against baseline methods in K-means clustering performance on the selected features. We vary the number of selected features $d \in \{16, 64, 128, 200, 600\}$. The comparison results are in Figure \ref{fig:clustering_performance}.
We follow the principles from its original paper to set hyper-parameters for each baseline method. We leave the discussion on model selection to latter sections and report the results of our method with the following parameter setting:
\begin{itemize}
\item We fix the composition ratio $\bar{\beta}=0.6$ based on observations in sections \ref{sec:hyperparameter_effect} and \ref{sec:solver_inspection}.
\item We set $\gamma$ via grid search in $\{0, 0.5, 1, \dots, 5\}$ while $nnz(\mathbf{r}) \geq d$. This is because overly strong $\gamma$ can force our method to generate too many absolutely zero entries in $\mathbf{r}$ to pick top $d$ features based on non-zero entries in $\mathbf{r}$.
\item We report the results of two block models amongst the $10$ candidates for each dataset. One is chosen for having the lowest RRE (i.e., $BMGUFS^*$ in Figure \ref{fig:clustering_performance}), thus selected before running our algorithm \ref{alg:bmgufs}. The other is selected via grid-search (i.e., $BMGUFS^s$ in Figure \ref{fig:clustering_performance}).
\end{itemize}

Figure \ref{fig:clustering_performance} demonstrates the superiority of our method over baselines in experiments. Specifically, we observe:
\begin{itemize}
\item Our BMGUFS selects features that achieve better clustering performance in basically all the investigated cases than the baselines. We constantly outperform our major competitor method NetFS in various settings.

We conjecture the superiority of our method on all the investigated datasets with the following explanations:
\begin{compactenum}
\item Block model of the structural graph provides more robust guidance against noisy links on real-world networks than detailed links and disconnections. Therefore our method outperforms the methods that use micro-level guidance (e.g., MMPOP \cite{wei2015efficient}).
\item Structural equivalence appreciated by block models is more informative than intra-community proximity between nodes to guide feature selection on the  investigated datasets. Therefore our method outperforms the methods guided by macro-level graph structure based on community analysis (e.g., NetFS \cite{li2016robust}).
\end{compactenum}
\item Our method achieves predominant ACC and NMI on each dataset at extremely small number of features. Specifically, our method outperforms the clustering results using all features by over $11 \%$ in ACC with only $d = 16$ features on BlogCatalog, whereas other baseline methods fail to surpass the performance of \textbf{all features} with such a small amount of features.
This highlights the power of our method in both finding high-quality features and dimension reduction.
\item The clustering performance of baseline methods that only consider the feature matrix, i.e., LapScore, SPEC, and NDFS, are consistently suboptimal to ours; however, they can outperform other baselines that incorporate graph structure in some cases. This supports the underlying premise of our work that \emph{block model can be a better way than other approaches to extract guidance from the structural graph for unsupervised feature selection}. %
\end{itemize}
\begin{figure*}
\begin{subfigure}{.33\linewidth}
\centering
\includegraphics[width=2.1in]{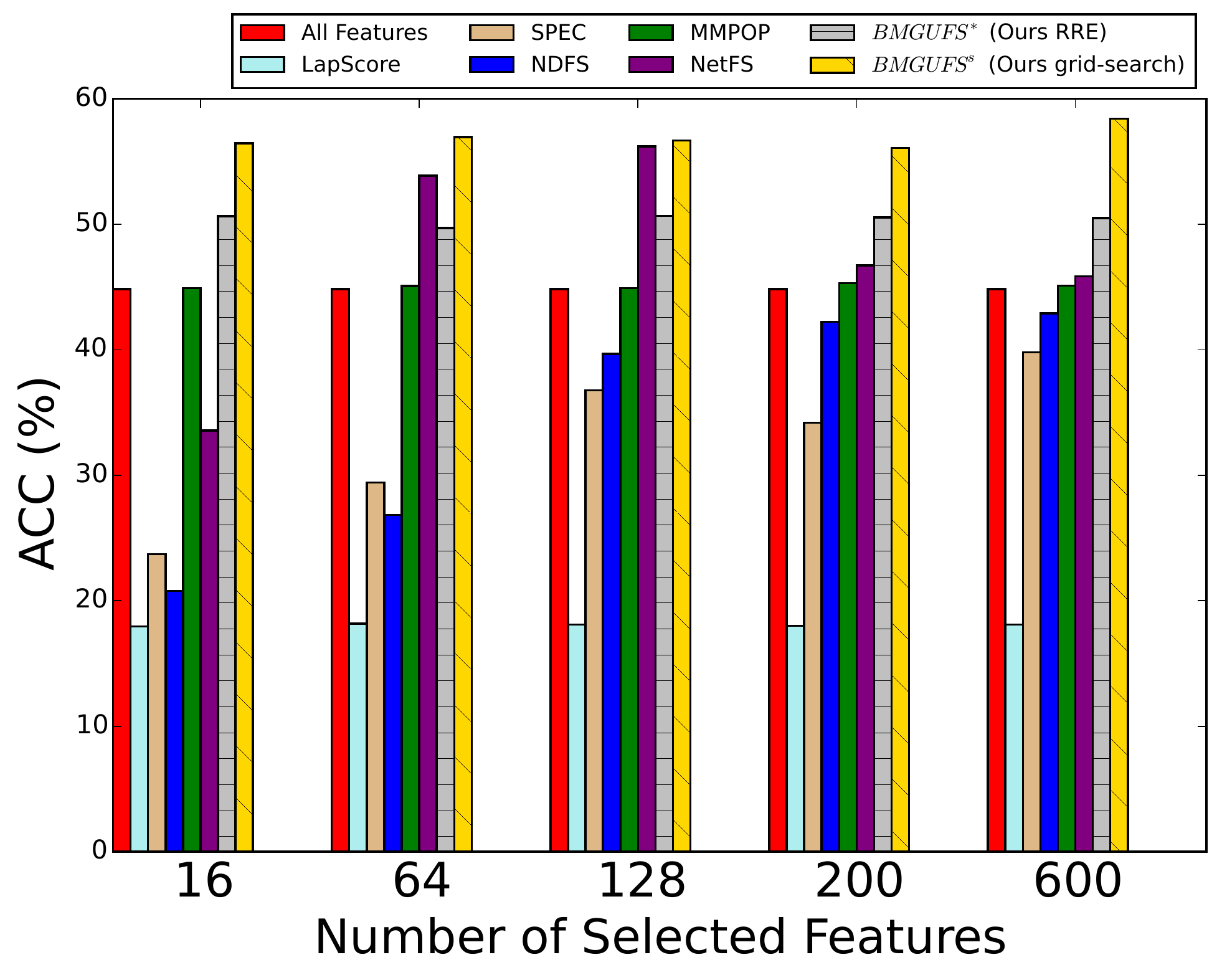}
\caption{BlogCatalog ACC}
\label{fig:blogcatalog_acc}
\end{subfigure}%
\begin{subfigure}{.33\linewidth}
\centering
\includegraphics[width=2.1in]{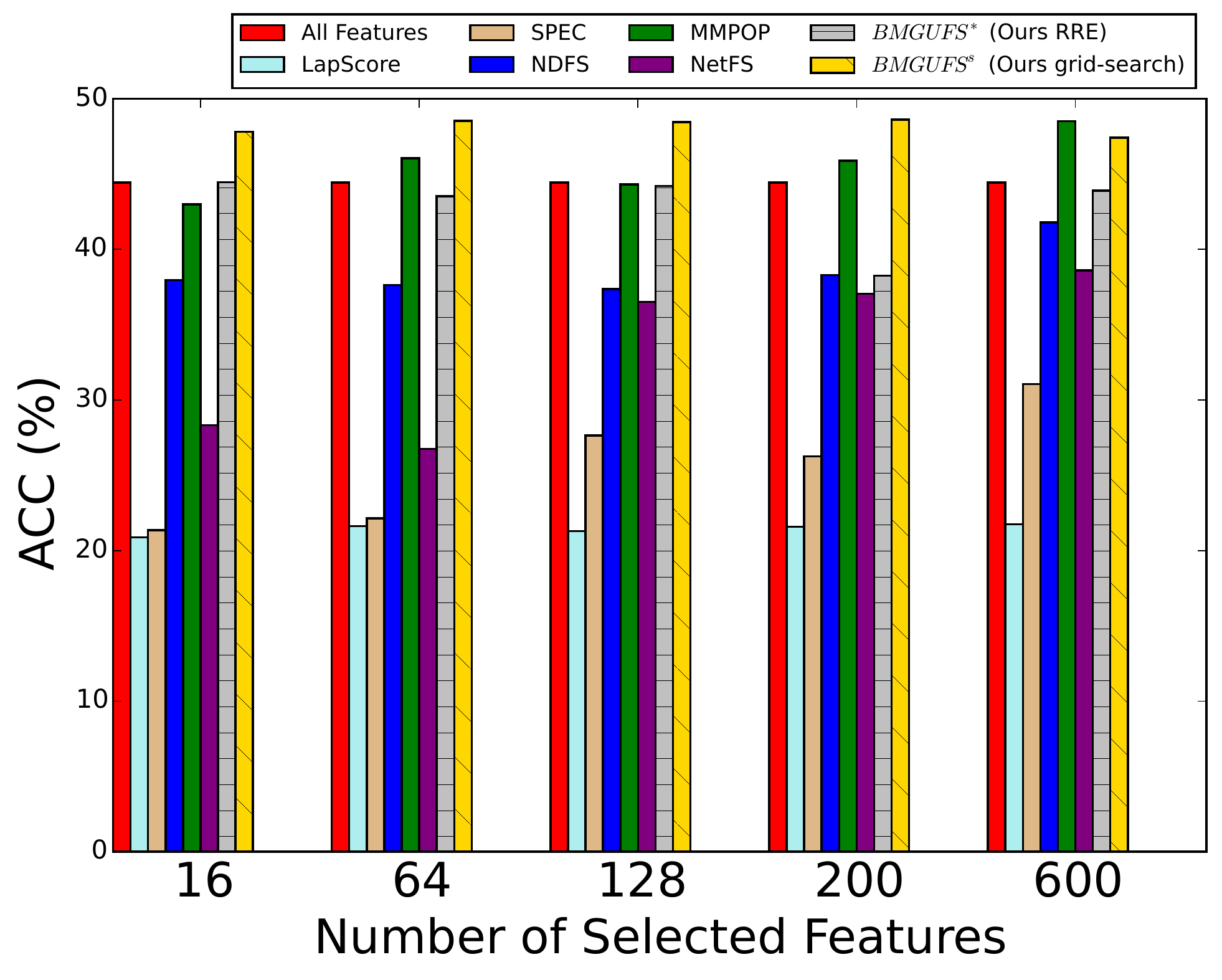}
\caption{Citeseer ACC}
\label{fig:citeseer_acc}
\end{subfigure}%
\begin{subfigure}{.33\linewidth}
\centering
\includegraphics[width=2.1in]{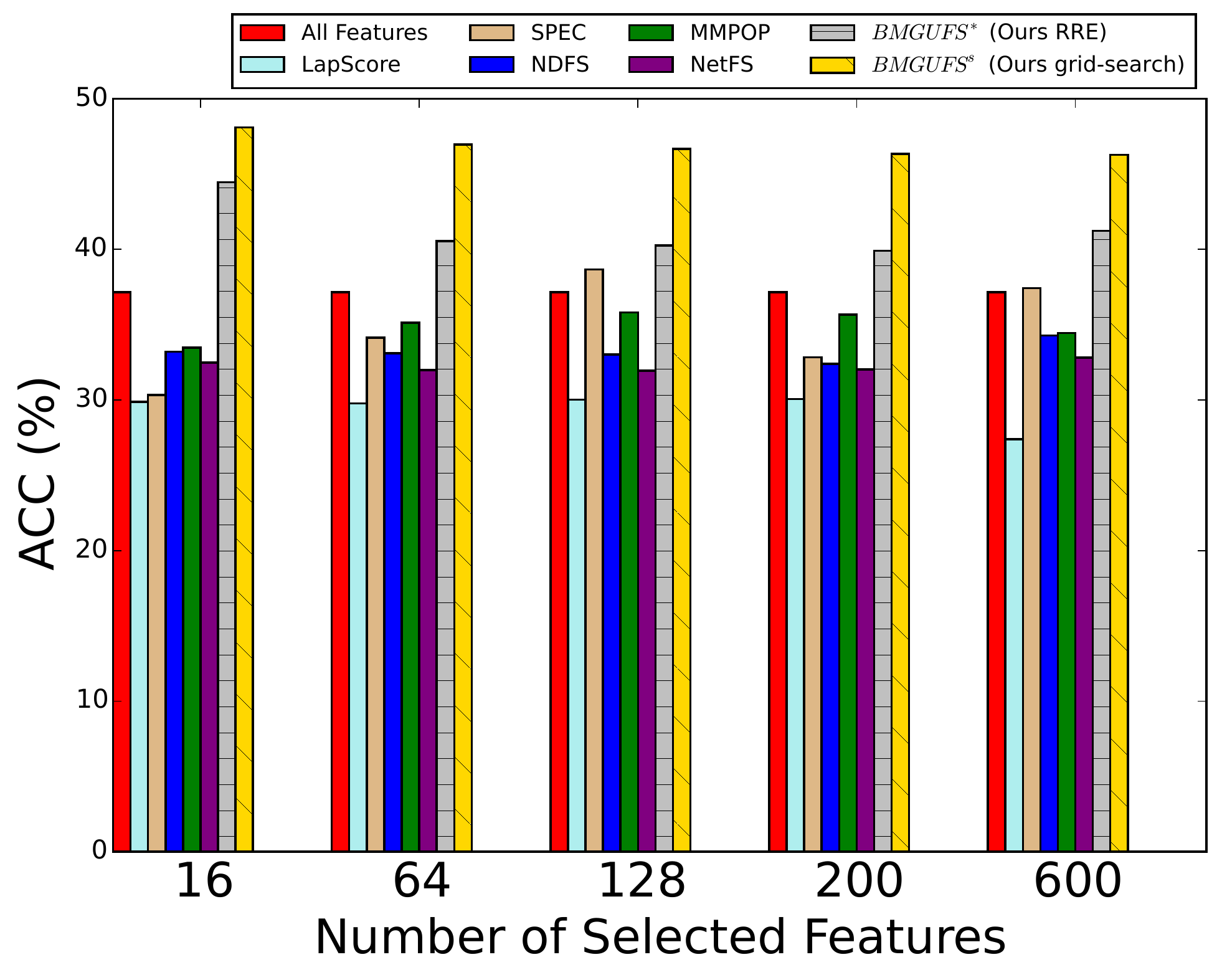}
\caption{Cora ACC}
\label{fig:cora_acc}
\end{subfigure}%
\\[1ex]
\begin{subfigure}{.33\linewidth}
\centering
\includegraphics[width=2.1in]{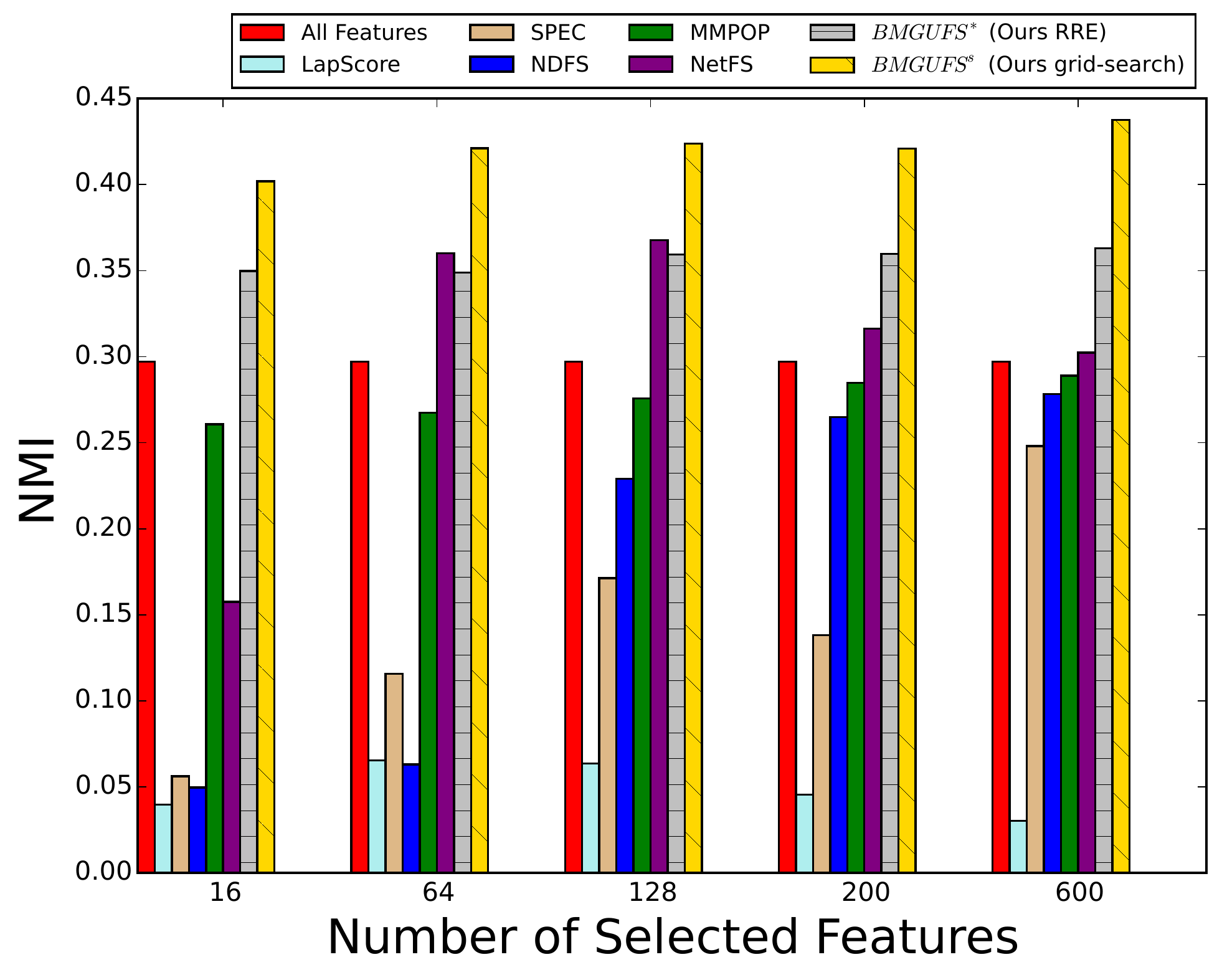}
\caption{BlogCatalog NMI}
\label{fig:blogcatalog_nmi}
\end{subfigure}%
\begin{subfigure}{.33\linewidth}
\centering
\includegraphics[width=2.1in]{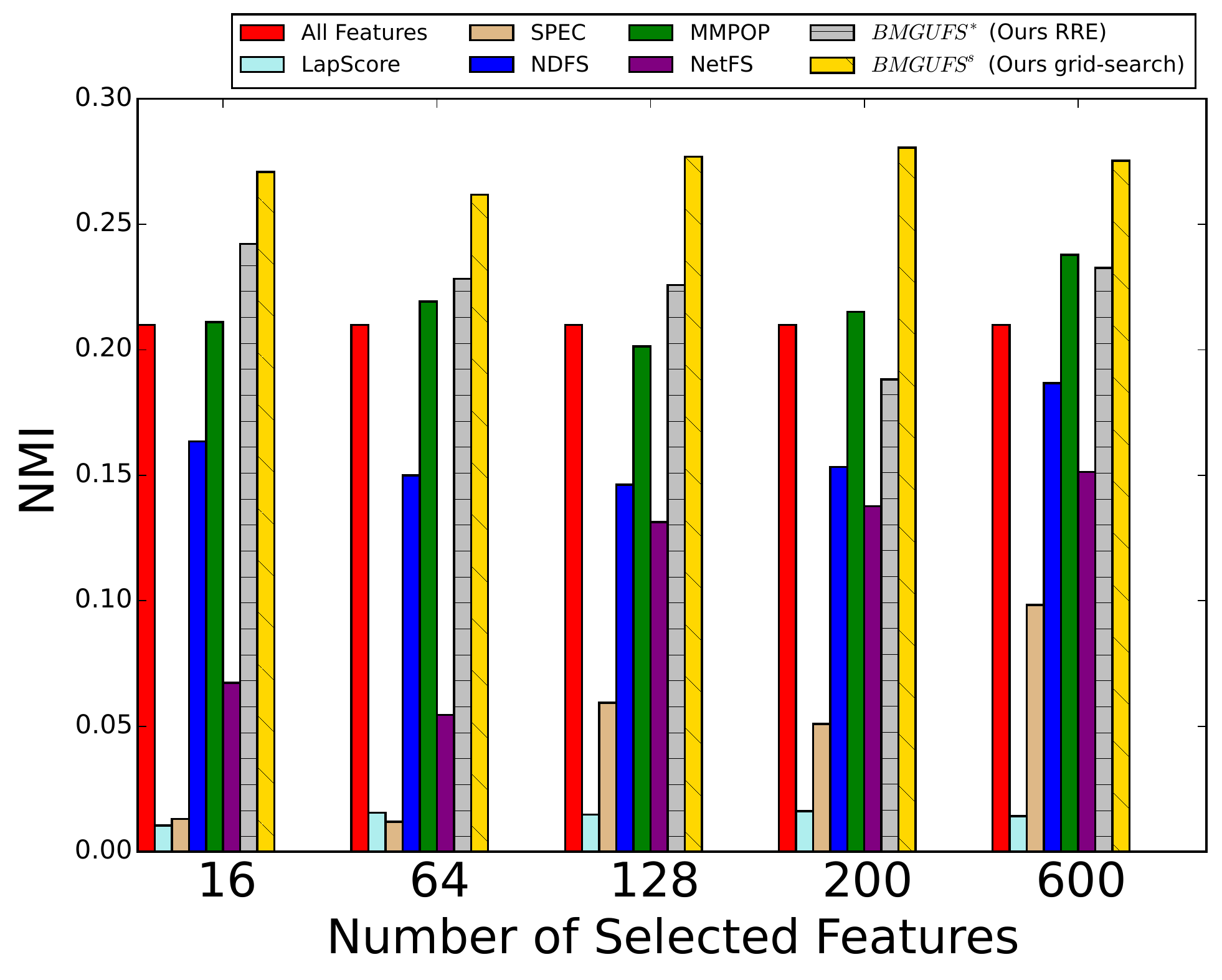}
\caption{Citeseer NMI}
\label{fig:citeseer_nmi}
\end{subfigure}%
\begin{subfigure}{.33\linewidth}
\centering
\includegraphics[width=2.1in]{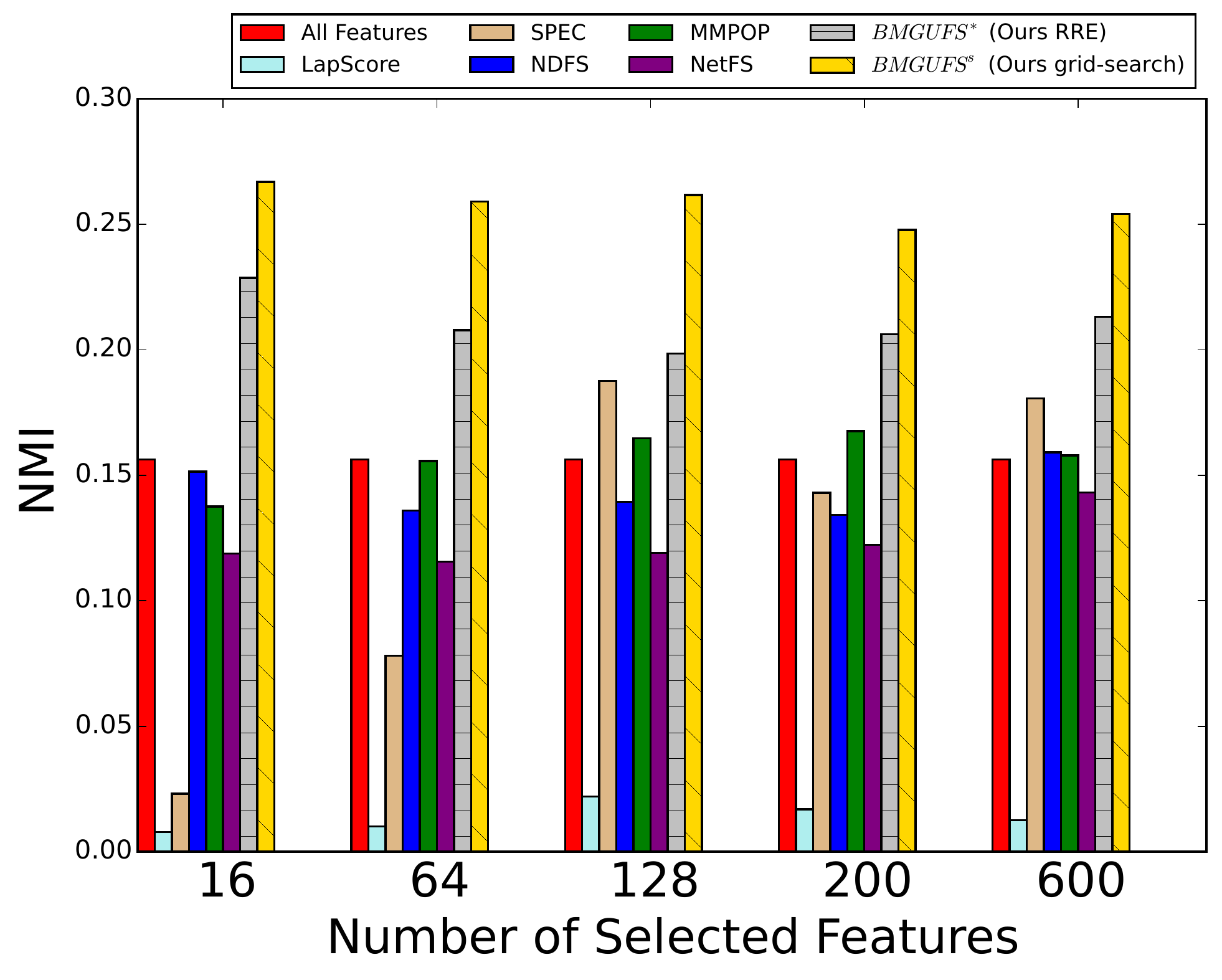}
\caption{Cora NMI}
\label{fig:cora_nmi}
\end{subfigure}%
\caption{Comparing K-means clustering performance of features selected by different unsupervised feature selection methods on different datasets. Top row: ACC ($\%$). Bottom row: NMI (in $[0,1]$). Datasets from left to right: BlogCatalog, Citeseer, Cora. Our method $BMGUFS^s$ with grid-search (the rightmost bars) achieves superior performance in basically all the cases. Our method $BMGUFS^*$ that uses the block model with the lowest RRE (the second rightmost bars) can attain highly-competitive results. Our method demonstrates advantage with extremely small amount of features (i.e., $d=16$).}\label{fig:clustering_performance}
\end{figure*}

\subsection{Sensitivity to Block Model Guidance}\label{sec:block_model_effect}
In this section we explore the sensitivity of our method w.r.t. the block model guidance from two perspectives: \textbf{(P1)} sensitivity to perturbations in the input block model and \textbf{(P2)} sensitivity to different block models generated from the same structural graph. We present the results on BlogCatalog as similar patterns exist on other datasets. We fix $\gamma=2$, $\bar{\beta} = 0.6$ for this section. Figure \ref{fig:bm_perturbation_sensitivity} shows our algorithm \ref{alg:bmgufs} is relatively robust against small perturbations in the input block model. Figure \ref{fig:block_model_sensitivity} shows the feature selection varies with different block models generated from the same graph, yet RRE is a reasonable criteria to select block model from multiple candidates before running our algorithm (see Figure \ref{fig:clustering_performance}).

\noindent \textbf{(P1)} We select the block model of lowest RRE as the base block model $\mathbf{F}$, $\mathbf{M}$. We introduce different levels of artificial perturbations (i.e., $5\%$ and $10\%$) by randomly selecting the given percentage of nodes and modifying their block memberships (i.e., random re-allocation).
We explore two situations: (a) only perturb $\mathbf{F}$ but keep the original $\mathbf{M}$ and (b) perturb $\mathbf{F}$ and recompute $\mathbf{M}$ for the structural graph $\mathbf{A}$. We measure the difference between the feature selection vectors $\mathbf{r}$ generated by the perturbed block models and $\mathbf{r}_0$ generated by the base block model with cosine distance, and summarize the results as box plots in Figure \ref{fig:bm_perturbation_sensitivity}.
We observe that our method is robust to small perturbations in block allocation. 

\noindent \textbf{(P2)} Different block models can be generated from the \emph{same} graph (e.g., from different local optima of equation \ref{eq:ONMtF_block_model}).
In Figure \ref{fig:block_model_sensitivity} (dark blue bars), we compare the K-means clustering performance on features selected by our method guided by the $10$ candidate block models generated from the structural graph of BlogCatalog. We observe notable variance in clustering performance on selected features of different block models. Nonetheless, the block model with the lowest RRE (e.g., $ID=3$ in Figure \ref{fig:block_model_sensitivity}) can guide our feature selection to highly-competitive (or even better) clustering performance in comparison to baselines (see $BMGUFS^*$ in Figure \ref{fig:clustering_performance}). This demonstrates the benefit of generating multiple block models as candidates to guide feature selection. We suggest to choose the block model with the lowest RRE as the guidance of our BMGUFS to alleviate grid-search in scenarios sensitive to computational cost.

We also explore whether (i) the quality of selected features is correlated with (ii) the accuracy of using block model to predict the ground-truth labels. We use the block allocation $\mathbf{F}$ as node clustering result to predict the ground-truth labels, and present the results in the yellow bars in Figure \ref{fig:block_model_sensitivity}. We observe no direct correlation between (i) and (ii).

\begin{figure}
\begin{subfigure}{.5\linewidth}
\centering
\includegraphics[width=1.8in]{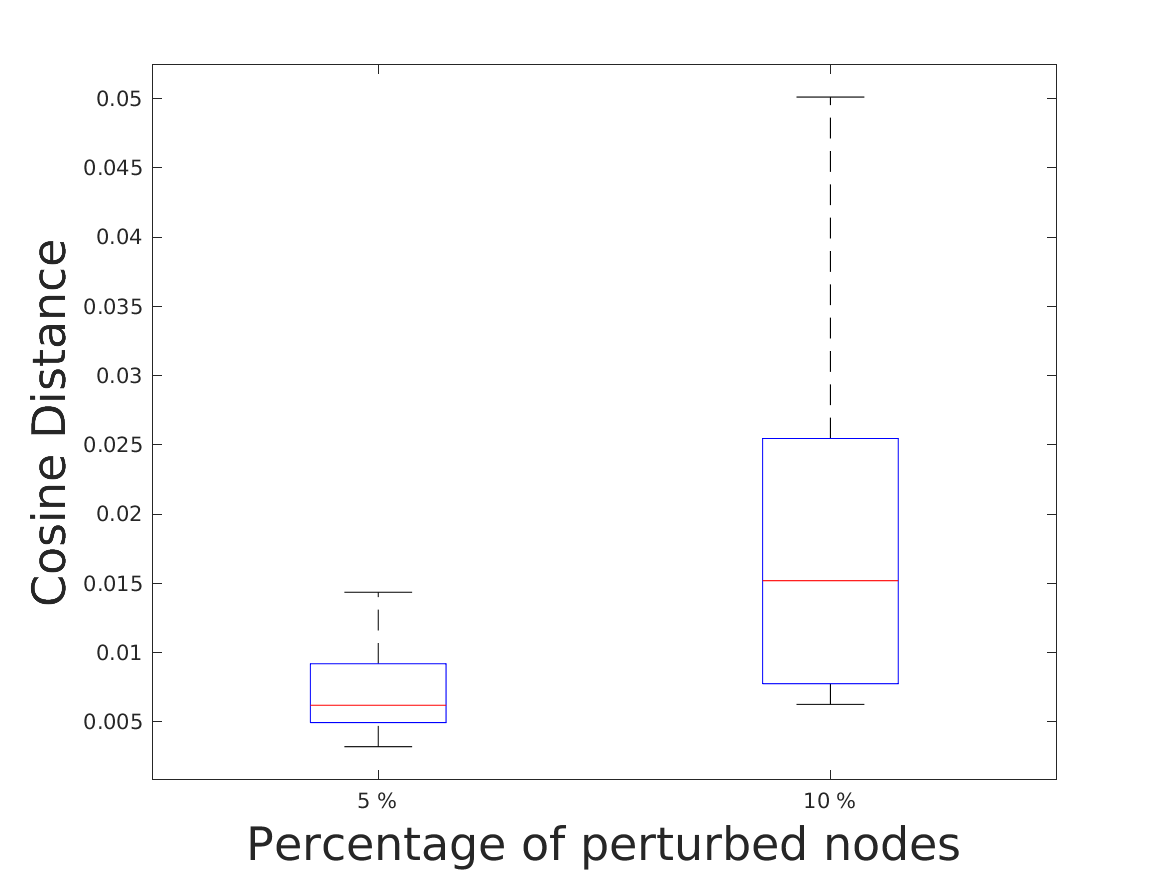}
\caption{Perturb $\mathbf{F}$, keep $\mathbf{M}$.}
\label{fig:overlaps_vs_block_models}
\end{subfigure}%
\begin{subfigure}{.5\linewidth}
\centering
\includegraphics[width=1.8in]{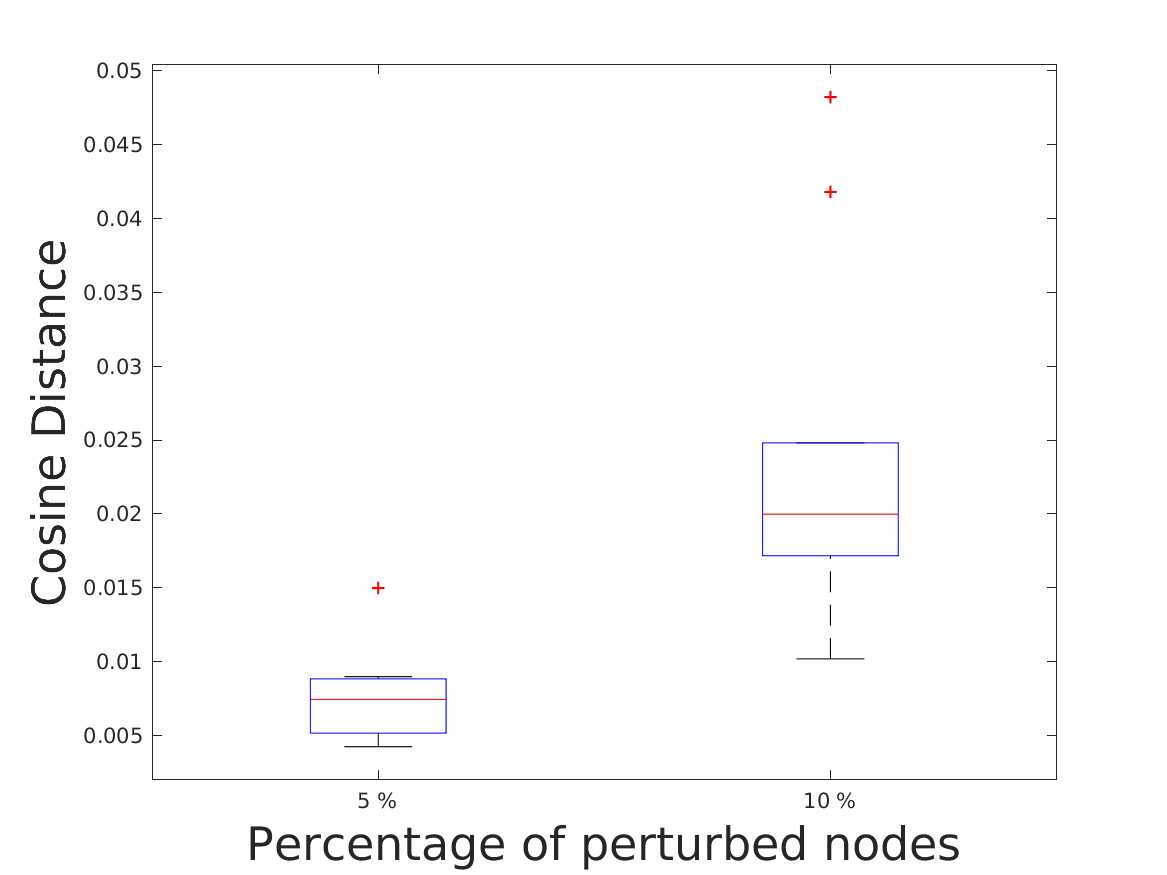}
\caption{Perturb $\mathbf{F}$, adjust $\mathbf{M}$.}
\label{fig:block_models_vs_acc}
\end{subfigure}%
\caption{Sensitivity analysis of BMGUFS w.r.t. different levels of perturbations in the input block model. Cosine distance measures the difference between $\mathbf{r}$ learnt from the perturbed block models and $\mathbf{r}_0$ of the block model without perturbation. \textbf{Left:} only perturb block allocation $\mathbf{F}$ with random re-allocation, keep the original image matrix $\mathbf{M}$. \textbf{Right:} perturb block allocation $\mathbf{F}$ and adjust $\mathbf{M}$ based on the structural graph. Our BMGUFS is robust to small perturbations in $\mathbf{F}$.} \label{fig:bm_perturbation_sensitivity}
\end{figure}

\begin{figure}
\begin{subfigure}{.5\linewidth}
\centering
\includegraphics[width=1.8in]{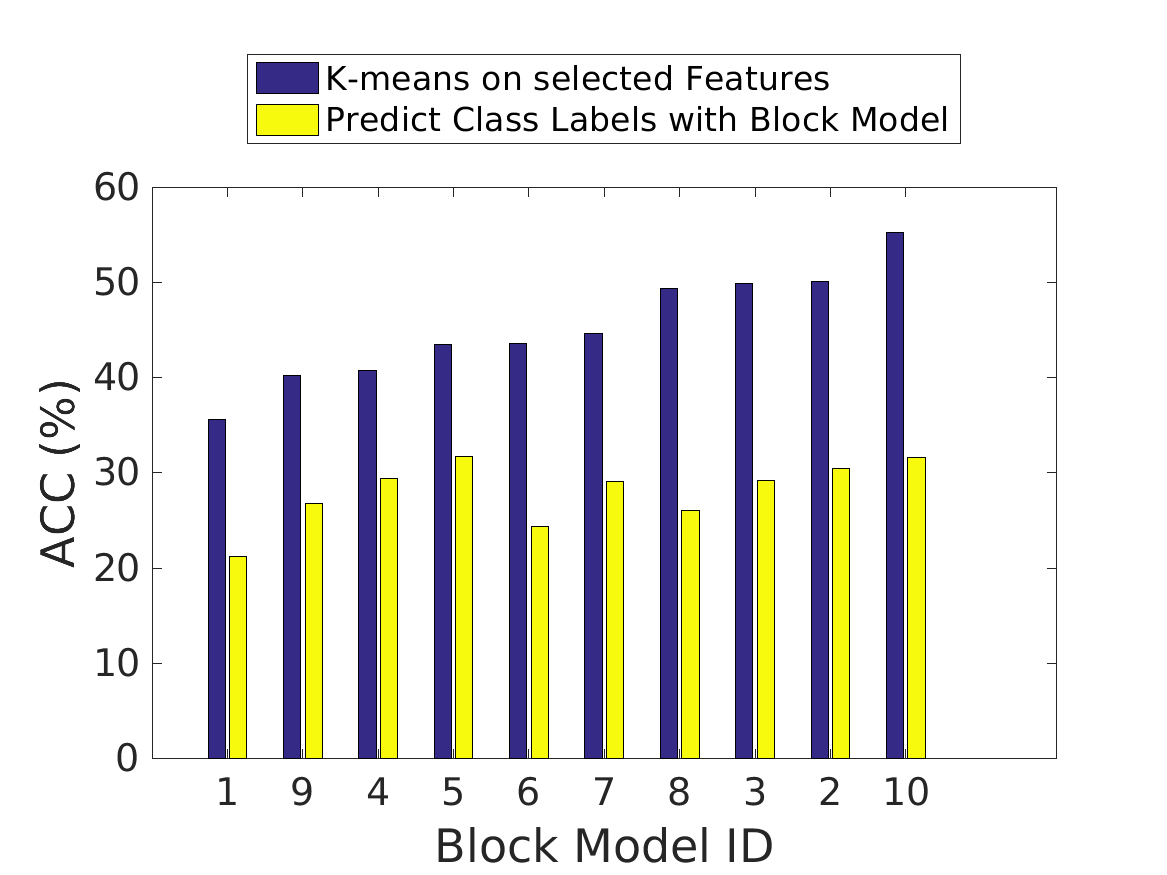}
\caption{ACC vs Block Models.}
\label{fig:block_models_vs_acc}
\end{subfigure}%
\begin{subfigure}{.5\linewidth}
\centering
\includegraphics[width=1.8in]{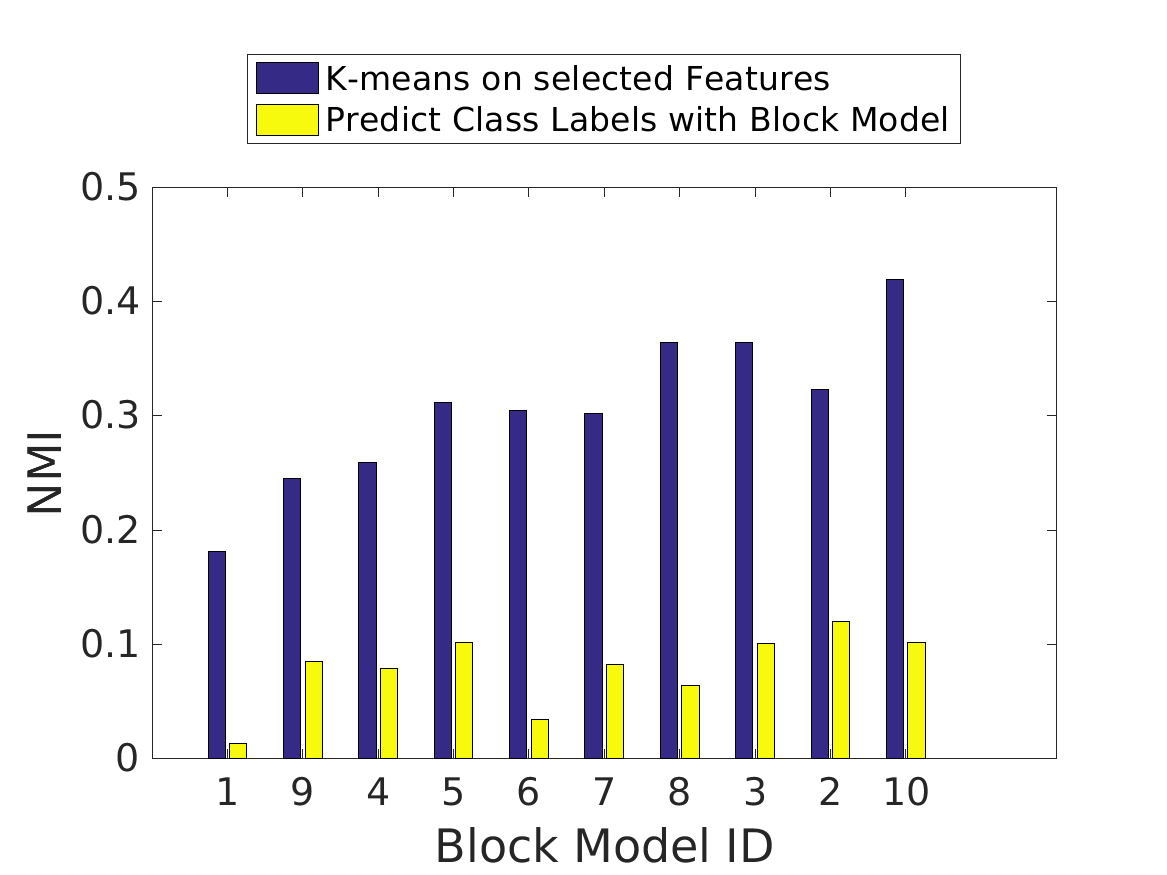}
\caption{NMI vs Block Models.}
\label{fig:overlaps_vs_block_models}
\end{subfigure}%
\caption{Comparing clustering performance on selected features guided by different block models (in blue bars). Clustering performance varies using \emph{different} block models generated from the \emph{same} structural graph. The prediction accuracy using block allocation (in yellow bars) does \emph{not} directly correlate to feature selection quality.} \label{fig:block_model_sensitivity}
\end{figure}
\subsection{Sensitivity to Parameter Selection}\label{sec:hyperparameter_effect}
We investigate the sensitivity of our BMGUFS to the two key hyper-parameters, composition ratio $\bar{\beta}$ and sparsity penalty $\gamma$. We vary $\bar{\beta} \in \{0, 0.1, 0.2, \dots, 0.9, 1 \}$ and $\gamma \in \{ 0, 0.5, 1, \dots, 4, 4.5, 5 \}$ in our experiments. We make two observations: (1) $\bar{\beta} = 0.6$ is preferred by various datasets and (2) it is practical to perform grid-search to find optimal $\gamma$.

In Figure \ref{fig:parameter_sensitivity}, we show the results of the top $d=16$ features selected by our method on BlogCatalog (first row) and Cora (second row) as representative examples. Citeseer shows similar pattern as BlogCatalog. We set $d = 16$ features to investigate a wider range of $\gamma$ as we witness stronger $\gamma$ can yield $\mathbf{r}$ with less non-zero entries than the required number of selected features. Nevertheless, our experiments cover cases where $\gamma$ is too large to ensure $nnz(\mathbf{r}) \geq d$. 
We report the clustering performance as $0$'s if $nnz(\mathbf{r}) < d$. \footnote{We avoid using zero entries in $\mathbf{r}$ in ranking the features as they are less informative than the non-zero ones.} 

Figure \ref{fig:parameter_sensitivity} shows a clear transition pattern around $\bar{\beta} = 0.5$ for both datasets. The clustering performance sustains at a relatively high level for all investigated $\bar{\beta} \geq 0.6$ when $nnz(\mathbf{r}) \geq d$. 
Varying $\gamma$ does not induce significant change in the clustering performance at $\bar{\beta} \geq 0.6$,  unless it is too large to sustain $nnz(\mathbf{r}) \geq d$. We acknowledge that different datasets favor different strength of sparsity (e.g, BlogCatalog favors strong sparsity yet Cora prefers none sparsity regulation). Note that the search space of $\gamma$ for our method is confined and bounded from both ends, because $\gamma \geq 0$ and $\gamma$ cannot be too large in order to sustain $nnz(\mathbf{r}) \geq d$. Therefore, it is practical to use grid search in practice to pursue better performance. 
\begin{figure}[t]
\begin{subfigure}{.5\linewidth}
\centering
\includegraphics[width=1.8in]{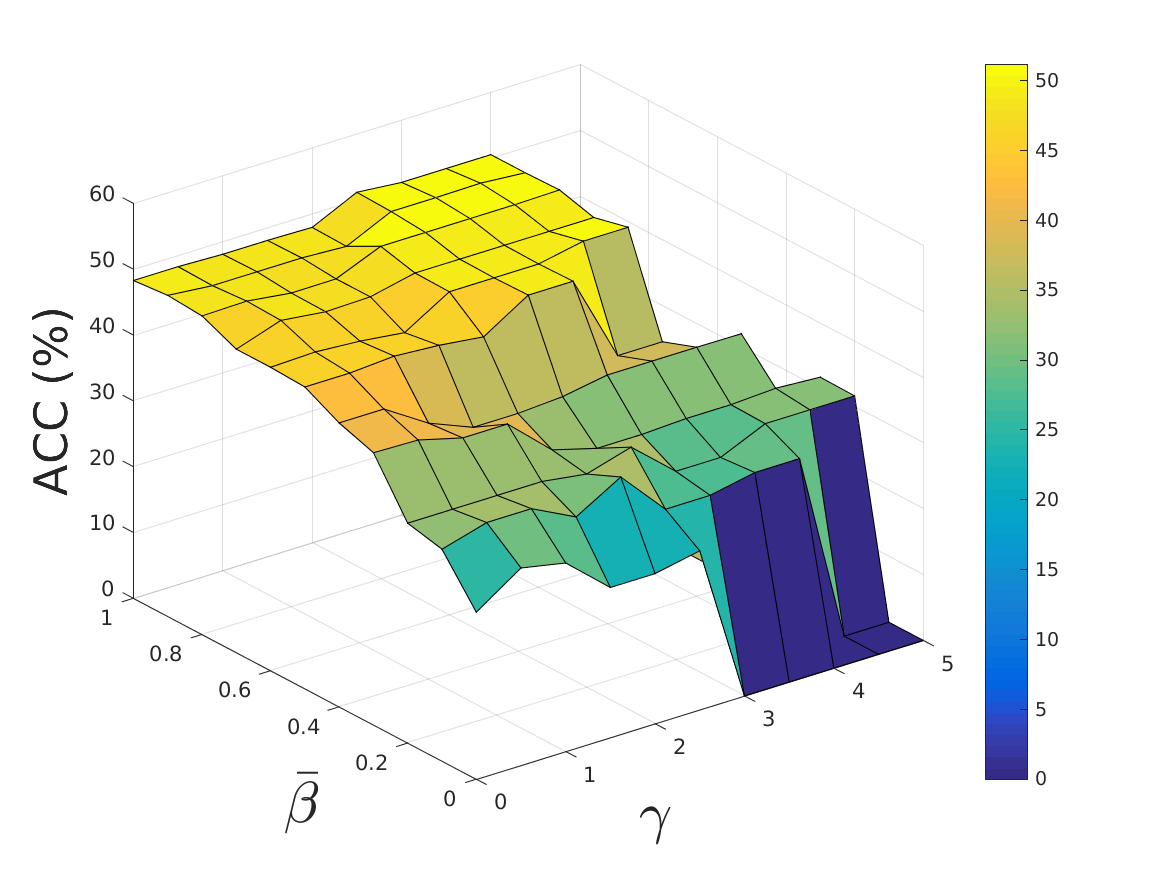}
\caption{BlogCatalog ACC}
\label{fig:blogcatalog_acc}
\end{subfigure}%
\begin{subfigure}{.5\linewidth}
\centering
\includegraphics[width=1.8in]{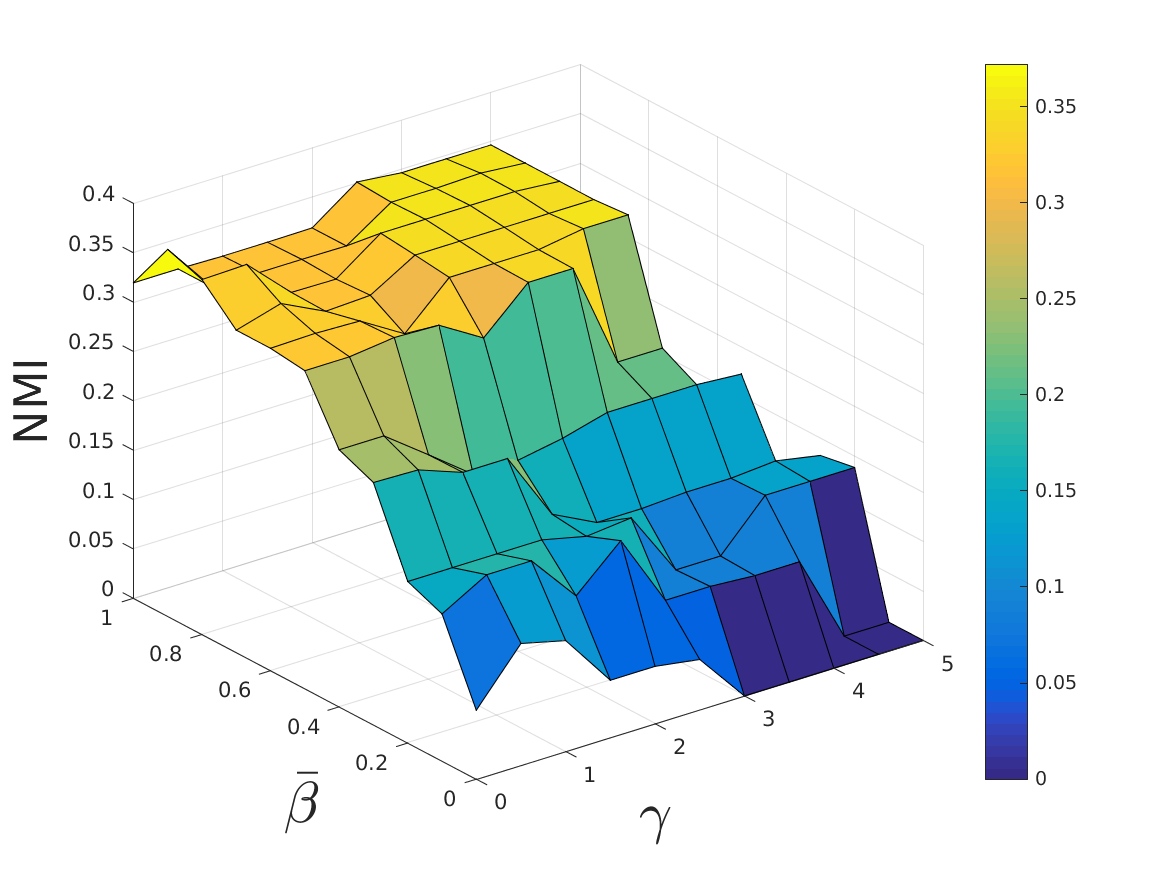}
\caption{BlogCatalog NMI}
\label{fig:flickr_acc}
\end{subfigure}%
\\
\begin{subfigure}{.5\linewidth}
\centering
\includegraphics[width=1.8in]{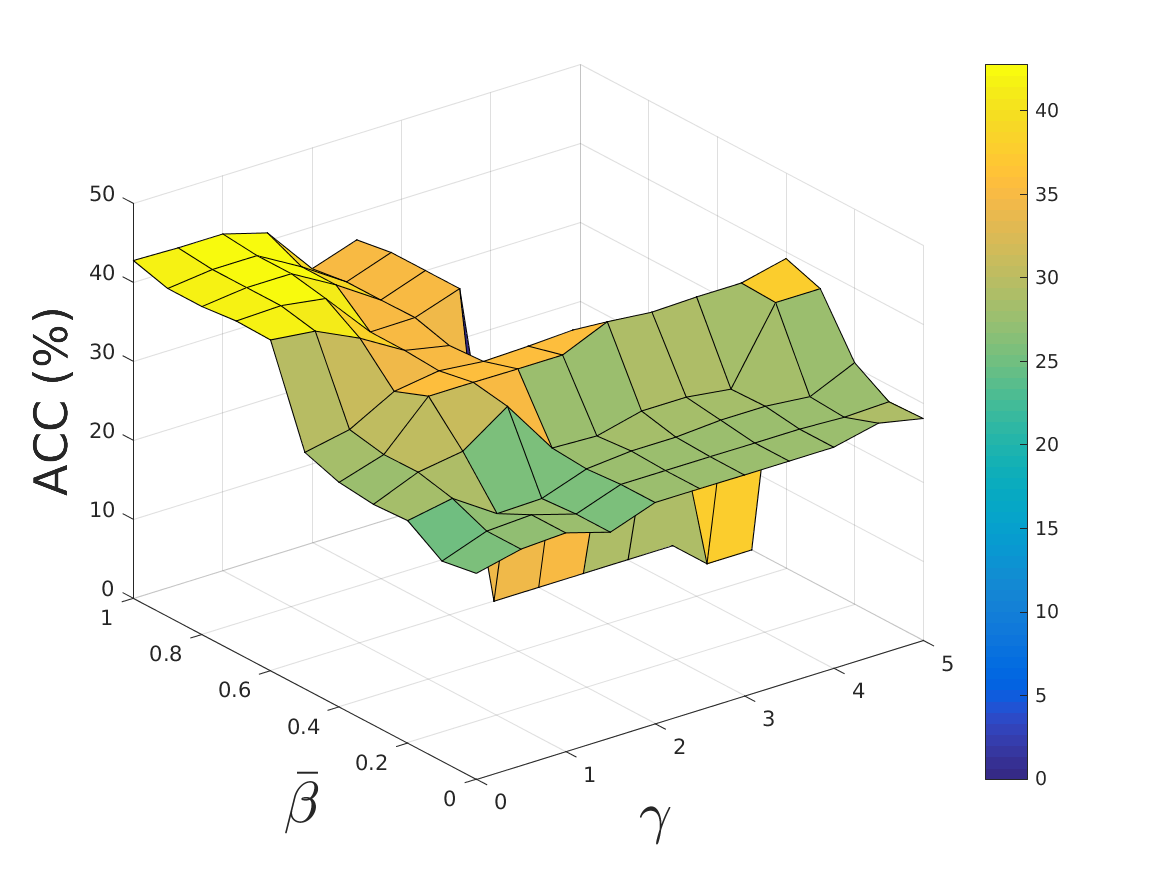}
\caption{Cora ACC}
\label{fig:blogcatalog_acc}
\end{subfigure}%
\begin{subfigure}{.5\linewidth}
\centering
\includegraphics[width=1.8in]{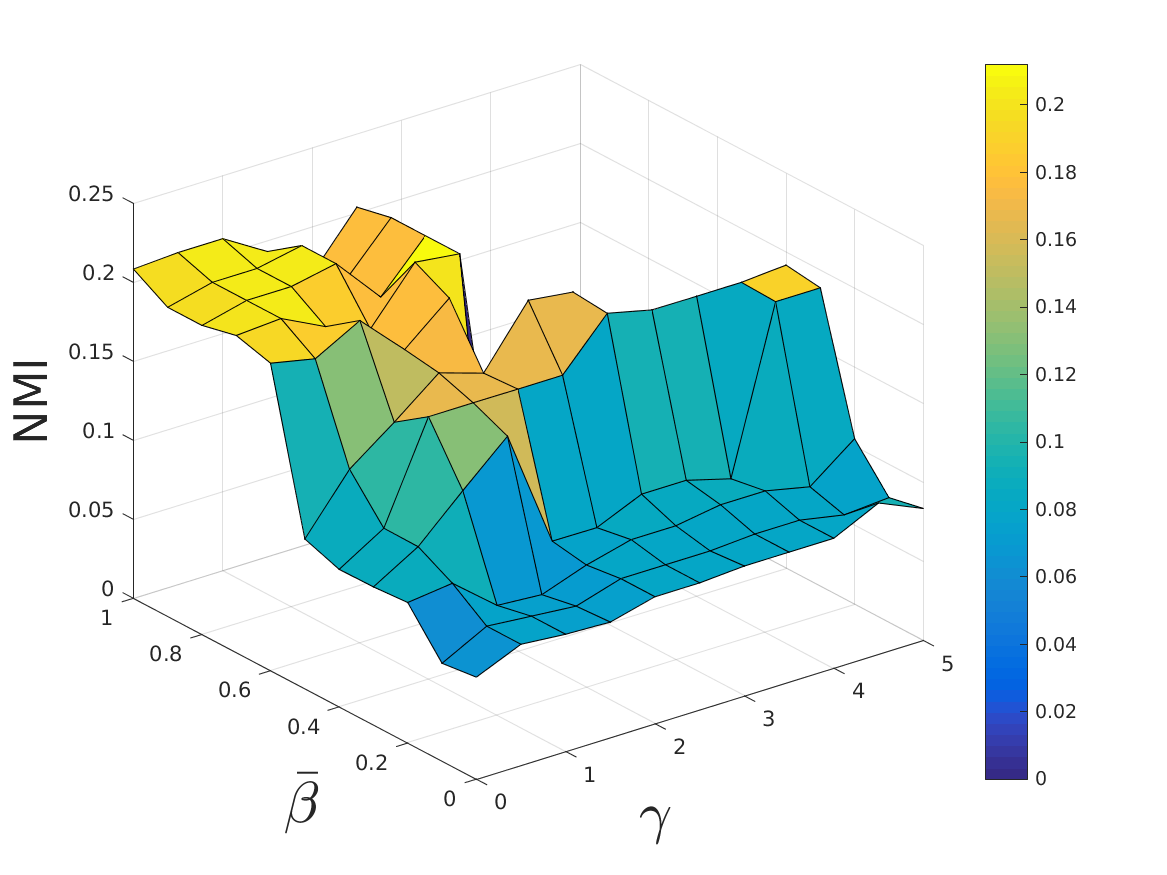}
\caption{Cora NMI}
\label{fig:flickr_acc}
\end{subfigure}%
\caption{Parameter sensitivity analysis w.r.t. $\bar{\beta}$ and $\gamma$ for K-means clustering on selected features. While $nnz(\mathbf{r}) \geq d$, the clustering performance is generally stable w.r.t. $\gamma$; $\bar{\beta}$ shows clear cliff effect around $\bar{\beta} = 0.5$ but remains stable at $\bar{\beta} \geq 0.6$. }\label{fig:parameter_sensitivity}
\end{figure}

\subsection{Solver Inspection}\label{sec:solver_inspection}
In this section we provide an empirical study on our heuristic solver. We inspect the values per iteration of our two objectives: $\mathcal{L}_b$, relative reconstruction error of block model on induced graph and, $\mathcal{L}_m$, the KL-divergence based distance measure between $\hat{\mathbf{M}}$ and the given $\mathbf{M}$.   As a representative example, we apply our BMGUFS at $\gamma=2$ on BlogCatalog guided by the block model with lowest RRE. We vary $\bar{\beta} \in \{0, 0.1, \dots, 1\}$ to demonstrate how it balances the decrement of the two objectives. Figure \ref{fig:combined_objective_vs_iterations} shows that at $\bar{\beta} \geq 0.6$ (solid curves), $\mathcal{L}_m$ monotonically decreases and $\mathcal{L}_b+\mathcal{L}_m$ converges within $200$ iterations. Interestingly, these $\bar{\beta}$ correlate with highly-competitive or state-of-the-art clustering performance in our experiments (see previous sections). This suggests the relative importance of optimizing $\mathcal{L}_m$ for our method to acquire high-quality features for clustering. 

\begin{figure}[t]
\begin{subfigure}{.5\linewidth}
\centering
\includegraphics[width=1.8in]{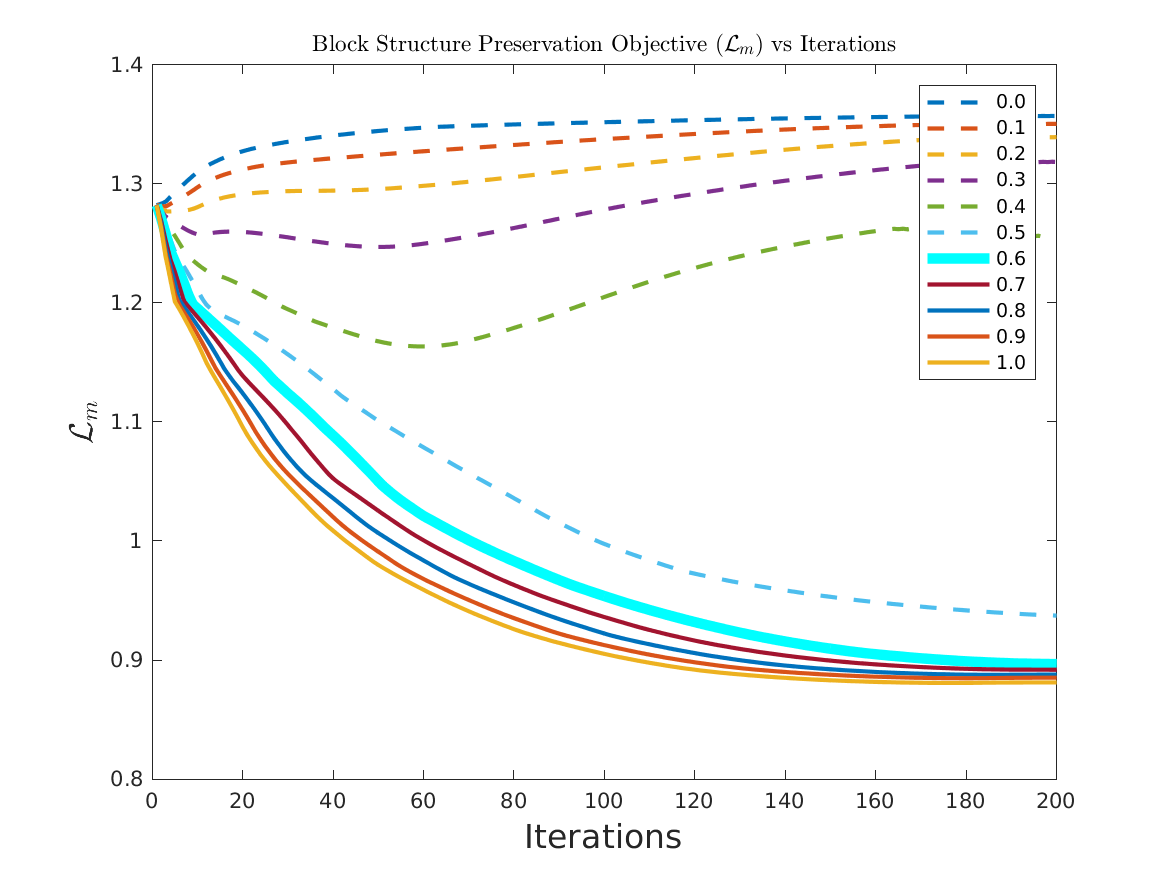}
\caption{$\mathcal{L}_m$ vs iterations.}
\label{fig:lm_vs_iterations}
\end{subfigure}%
\begin{subfigure}{.5\linewidth}
\centering
\includegraphics[width=1.8in]{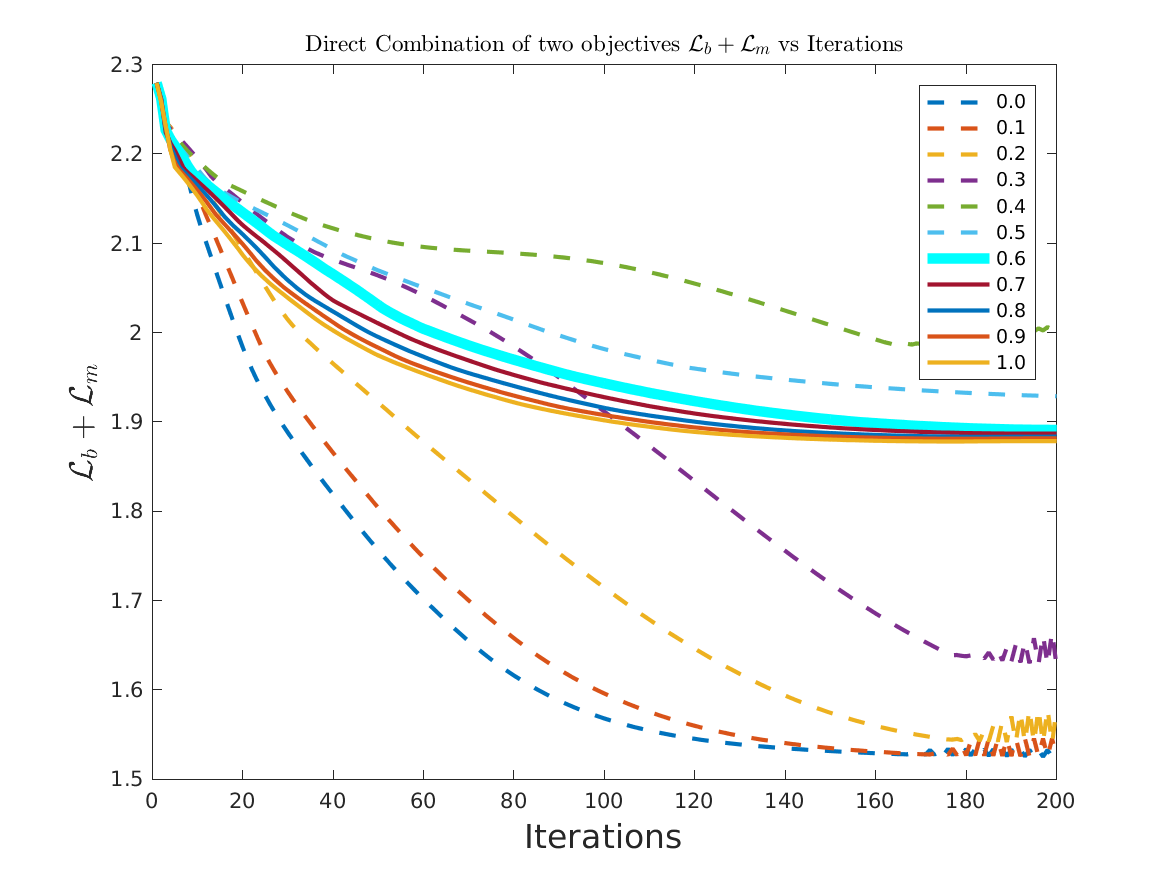}
\caption{$\mathcal{L}_b + \mathcal{L}_m$ vs iterations.}
\label{fig:lb_and_lm_vs_iterations}
\end{subfigure}%
\caption{Inspection on the objective function per iteration with varying $\bar{\beta}$. \textbf{Left:} $\mathcal{L}_m$ vs iterations. \textbf{Right:} $\mathcal{L}_b+\mathcal{L}_m$ vs iterations. When $\bar{\beta} \geq 0.6$, $\mathcal{L}_m$ and $\mathcal{L}_b + \mathcal{L}_m$ both decrease smoothly and monotonically, and converge within $200$ iterations. Compared to Figure \ref{fig:parameter_sensitivity}, this suggests the importance of ensuring the minimization of $\mathcal{L}_m$ for our method to acquire high-quality features for clustering.}\label{fig:combined_objective_vs_iterations}
\end{figure}

\section{Related Work}\label{sec:related_work}
In this section we briefly review related work on unsupervised feature selection 
and discuss applications of 
block model.

\noindent \emph{\textbf{Unsupervised Feature Selection.}}
There have been a lot of efforts for improving unsupervised feature selection performance. Many different methods have been proposed for solving different problems, e.g. adding $l_{2,1}$-norm minimization constraint to reduce the redundant or even noisy features \cite{yang2011l2, li2012unsupervised}, algorithms for improving the robustness of graph embedding and sparse spectral regression \cite{shi2014robust}, various methods for adaptive structure learning in different scenarios \cite{du2015unsupervised, li2019adaptive}, etc.  Concerning different selection strategies, feature selection methods can be broadly categorized \cite{li2017feature} as wrapper (e.g., \cite{kohavi1997wrappers, dy2004feature}), filter (e.g., \cite{he2006laplacian, zhao2007spectral, li2012unsupervised}), and embedded methods (e.g., \cite{hou2011feature, qian2013robust, li2016robust,  belkin2002laplacian}).

The concept of graph has long been explored in unsupervised feature selection to extract proximal supervisory signals \cite{li2017feature}, including seminal work based on spectral graph analysis (e.g., LapScore \cite{he2006laplacian}, SPEC \cite{zhao2007spectral}, and NDFS \cite{yang2011l2}). However, this line of work has mostly been focused on utilizing the topological patterns in the \emph{original feature space} as regularization or constraints for feature selection (e.g., \cite{gu2012locality, du2013local, shi2014robust}). 

It is relatively recent to incorporate \emph{another} network over instances to guide feature selection with the emergence of attributed networks in various domains (e.g., \cite{taxidou2014online, safari2014protein}). This provides another source of guidance based on the intuition that the links between instances indicates their similarity in the selected feature space. We contextualize our work in this area. 
The work most closely related to ours can be generally categorized based on the scope of patterns explored on the network topology into (1) micro-level \cite{wei2015efficient, wei2016unsupervised, li2019adaptive} and (2) macro-level \cite{tang2012unsupervised, li2016robust}. 
A common limitation of directly incorporating links and disconnections as \emph{micro-level} guidance is to be susceptible to noisy and incomplete links, which commonly exist in real-world large and complex networks \cite{liu2016good}. Our method explores block model as a \emph{macro-level} structural guidance to alleviate such issues.

As for \emph{macro-level (community analysis)}, one of the earliest efforts on network-guided unsupervised feature selection LUFS \cite{tang2012unsupervised} extracts social dimensions as non-overlapping clusters of nodes to regularize feature selection. NetFS \cite{li2016robust} embeds the latent representation learning into the feature selection process, which has been reported to be highly effective on various public datasets. However, the unified optimization framework of NetFS \cite{li2016robust} can be influenced by the low-quality features, hence the guidance can deviate from the desirable macro-level network structure. Moreover, their SymNMF-based latent representation learning formulation is consistent with community analysis. 
Our work is fundamentally different from existing related work both model-wise and methodology-wise.

\noindent \emph{\textbf{Block Model}} is a popular method to group nodes based on structural equivalence/similarity \cite{muller2012neural}, which has been relaxed to stochastic equivalence \cite{abbe2017community}. 
It leads to the well-studied area of Stochastic Block Model (SBM), which has long been correlated to community detection \cite{abbe2017community}. 
There are studies on the equivalence between NMF-based formulations and SBM at their respective optima \cite{paul2016orthogonal, zhang2018equivalence}. 
To model disassortative block-level interactions and other challenging structural patterns beyond conventional community structure \cite{aicher2015learning}, \cite{ganji2018image} formulates block model discovery with explicit structural constraints on the image matrix. 
Recent advances in block model include applications to a wide range of challenging domains, e.g., brain imaging data \cite{bai2017unsupervised, bai2018mixtures}  and Twitter \cite{bai2018discovering}. Methodologically, \cite{mattenet2019generic} proposes a framework to efficiently solve block allocation in binary with constraints on the image matrix based on Constrained Programming (CP). 

In this paper we use the multiplicative update rules developed by seminal work \cite{ding2006orthogonal} to build block models from the structural graph.
To the best of our knowledge, we are the first to explore modeling the feature selection problem as graph learning regularized by a block model precomputed from the structural graph. 
There have been many kinds of block models \cite{abbe2017community} with different characteristics that can potentially match unsupervised feature selection for different linked data.
We envision it to facilitate many powerful unsupervised feature selection methods for various complicated real-world attributed networks. 

\section{Conclusion}\label{sec:conclusion}
We propose a novel graph-driven unsupervised feature selection method guided by the block model.
The block modeling process in our method is not influenced by the original feature set. 
Moreover, our similarity graph over nodes in the selected feature space explicitly exploits the relations between nodes. 
Methodologically, we not only utilize the grouping of nodes as blocks but also require the induced graph to respect the precomputed image matrix. Hence, our method utilizes more complex macro-level network structure beyond conventional community structures.
Experiments on various real-world datasets demonstrate the effectiveness of our method as it outperforms baseline methods in finding high-quality features for K-means clustering. 
For in-depth characterization of our method, we explored the sensitivity of our method regarding hyper-parameters and different block models generated from the graph. This can facilitate the application of our method to domains beyond our exploration. 
Our method demonstrates the power of using block model to guide unsupervised feature selection. 
Methodologically, we leave it as future endeavor to jointly learn block model for the graph and select high-quality features.


\begin{acks}
This research is supported by the National Science Foundation via grant IIS-1910306 and ONR via grant N000141812485.
\end{acks}

\bibliographystyle{ACM-Reference-Format}
\bibliography{bmgufs-bibliography}

\newpage

\appendix

\section{Proof for Theorem \ref{theorem:closed_form_M}}\label{sec:proof}
\begin{theorem*}[Least Squares Optimal $\mathbf{M}$ in Closed Form]
Given $\mathbf{A} \in R^{n \times n}$, $\mathbf{F} \in \{0,1\}^{n\times k}$.  If $\mathbf{D} = \mathbf{F}^T \mathbf{F}$ is a diagonal matrix with positive diagonal elements, then 
\begin{equation}
\underset{\mathbf{X}}{argmin} \| \mathbf{A} - \mathbf{F} \mathbf{X} \mathbf{F}^T \|^2_F = \mathbf{D}^{-1}\mathbf{F}^T \mathbf{A} \mathbf{F} \mathbf{D}^{-1}
\end{equation}
\end{theorem*}

\begin{proof}
Let $\mathbf{M} = \underset{\mathbf{X}}{Argmin} \| \mathbf{A} - \mathbf{F} \mathbf{X} \mathbf{F}^T \|^2_F$. Then the following equations exist by vectorizing matrices $\mathbf{M}$ and $\mathbf{A}$:
\begin{equation}
\begin{aligned}
vec(M) & =  \underset{\mathbf{x}}{argmin} \| vec(\mathbf{A}) - (\mathbf{F} \otimes \mathbf{F}) \mathbf{x} \|^2_2 \\
& = [(\mathbf{F} \otimes \mathbf{F})^T(\mathbf{F} \otimes \mathbf{F})]^{-1} (\mathbf{F} \otimes \mathbf{F})^T vec(\mathbf{A}) \\ 
& = [(\mathbf{F}^T \mathbf{F})  \otimes (\mathbf{F}^T \mathbf{F})]^{-1} (\mathbf{F} \otimes \mathbf{F})^T vec(\mathbf{A}) \\
& = [\mathbf{D} \otimes \mathbf{D}]^{-1} (\mathbf{F} \otimes \mathbf{F})^T vec(\mathbf{A}) \\ 
& \overset{(*)}{=} [\mathbf{D}^{-1} \otimes \mathbf{D}^{-1}] (\mathbf{F}^T \otimes \mathbf{F}^T) vec(\mathbf{A}) \\
& = [(\mathbf{D}^{-1} \mathbf{F}^T) \otimes (\mathbf{D}^{-1} \mathbf{F}^T) ] vec(\mathbf{A}) \\
\end{aligned}
\end{equation}
where equation (*) sustains because $\mathbf{D} = \mathbf{F}^T \mathbf{F}$ is a diagonal matrix with positive diagonal elements, hence $\mathbf{D}$ is non-singular.
Lastly, rewrite $\mathbf{M}$ and $\mathbf{A}$ into matrices:
\begin{equation}
\mathbf{M} = \mathbf{D}^{-1} \mathbf{F}^T \mathbf{A} \mathbf{F} \mathbf{D}^{-1}
\end{equation}
\end{proof}

\section{Details on experiments}\label{sec:experimental_details}
\subsection{Datasets Preprocessing}\label{sec:dataset_preprocessing}
\begin{itemize}
\item \textbf{BlogCatalog} \cite{tang2009relational, Huang-etal17Label} BlogCatalog is a social network dataset, which contains an undirected graph of $5196$ nodes and $171743$ edges, where each node is a BlogCatalog user, and each edge indicates whether the two users are friends on the website. The attributes are the appearance of $8189$ keywords in the blog descriptions. We obtain the dataset from \url{http://people.tamu.edu/~xhuang/BlogCatalog.mat.zip}.
\item \textbf{Citeseer} \cite{sen2008collective} Citeseer is a citation network dataset, which contains a directed graph of $3327$ nodes where each edge indicates a citation of a paper to another, as well as the appearance of $3703$ words in $3312$ papers out of $3327$ papers in the graph. We obtain the dataset from \url{https://github.com/tkipf/gcn} \cite{kipf2016semi}, and we cross-check the data set with the data set obtained from \url{https://linqs-data.soe.ucsc.edu/public/lbc/citeseer.tgz}. The citation network in the dataset is a directed graph. We transform the graph to an undirected graph by creating $(u, v)$ and $(v, u)$ edges for every edge $(u, v)$ in the directed graph. There are $15$ papers that do not have attributes, we remove the corresponding nodes and edges from the graph. As a result, the Citeseer dataset we use in the experiment contains a citation graph of $3312$ nodes and $4660$ undirected edges.
\item \textbf{Cora} \cite{sen2008collective} Cora is a citation network dataset, which contains a directed graph of 2708 nodes where each edge indicates a citation of a paper to another, and appearance of $1433$ words of all papers in the graph. We obtain the dataset from \url{https://github.com/tkipf/gcn} \cite{kipf2016semi}, and we cross-check the data set with the one obtained from \url{https://linqs-data.soe.ucsc.edu/public/lbc/cora.tgz}. The citation network in the dataset is a directed graph. We transform the graph to an undirected graph by creating $(u, v)$ and $(v, u)$ edges for every edge $(u, v)$ in the directed graph. As a result, the Cora dataset we use in the experiment contains a citation graph of $2708$ nodes and $5278$ undirected edges.
\end{itemize}

\subsection{K-means Clustering} For our experiments, we use
scikit-learn's implementation of K-means clustering algorithm. We also
use sklearn's normalize function to normalize the feature sets before
running the clustering algorithm \cite{pedregosa2011scikit}. For all
experiments involving K-means clustering, we run the algorithm 20 times
to compensate for random initialization.

\subsection{Definitions of Metrics for Clustering Performance Evaluation}\label{appendix:metrics}
\noindent \textbf{Accuracy (ACC)}  
\begin{equation}\label{eq:acc}
ACC = \frac{\underset{i \in [n]}{\Sigma} \mathcal{I}(l_i = \sigma (c_i))}{n}
\end{equation}
where $c_i$ is the clustering result of data point $i$ and $l_i$ is its ground-truth class label. Permutation function $\sigma$ maps $c_i$ to a class label using Kuhn-Munkres Algorithm. We utilize the implementation of ACC from \url{https://github.com/Tony607/Keras_Deep_Clustering/blob/master/metrics.py} in our experiments.

\noindent \textbf{Normalized Mutual Information (NMI)}
\begin{equation}\label{eq:nmi}
NMI(\mathcal{L}, \mathcal{C}) = \frac{MI(\mathcal{L}, \mathcal{C})}{max(H(\mathcal{L}), H(\mathcal{C}))}
\end{equation}
where $H(\mathcal{L})$ and $H(\mathcal{C})$ are respectively the entropy of $\mathcal{L}$ (the node/instance grouping based on class labels) and $\mathcal{C}$ (the instance clustering based on selected features) respectively.

\subsection{Baseline Methods: Links to Source Codes and Setting Hyper-paramters}\label{sec:baseline_codes_hp}
\begin{itemize}
        \item \textbf{MMPOP} \cite{wei2015efficient} 
        \begin{itemize}
        \item \textbf{Link to Source Code:} \url{http://www.cse.lehigh.edu/~sxie/codes/optimization_pop.py}. 
        \item \textbf{Setting parameters:} We keep the default parameters across all experiments.
\end{itemize}
        \item \textbf{NetFS} \cite{li2016robust} 
        \begin{itemize}
        \item \textbf{Link to Source Code:} \url{http://people.virginia.edu/~jl6qk/code/NetFS.zip}. 
        \item \textbf{Setting parameters:} We set $\alpha$ to be $10$ and $\beta$ to be 0.1 according to the original paper's experiments.
        \end{itemize}
        \item \textbf{NDFS} \cite{li2012unsupervised} 
        \begin{itemize}
        \item \textbf{Link to Source Code:} We use skfeature \cite{li2017feature}'s implementation of Nonnegative Discriminative Feature Selection algorithm from \url{https://github.com/jundongl/scikit-feature/blob/master/skfeature/example/test_NDFS.py}
        \item \textbf{Setting parameters:} We utilize the default setting in skfeature's provided example. The number of clusters is set according to the number of distinct ground truth labels in each dataset. The number of neighbors is set to $5$ as suggested in the original paper.
        \end{itemize}
        \item \textbf{SPEC} \cite{zhao2007spectral} 
        \begin{itemize}
        \item \textbf{Link to Source Code:} We use skfeature \cite{li2017feature}'s implementation of Spectral Feature Selection algorithm from \url{https://github.com/jundongl/scikit-feature/blob/master/skfeature/example/test_SPEC.py} 
        \item \textbf{Setting parameters:} We utilize the default setting in skfeature's provided example. The number of clusters is set according to the number of distinct ground truth labels in each dataset.
        \end{itemize}
        \item \textbf{Laplacian Score} \cite{he2006laplacian} 
        \begin{itemize}
        \item \textbf{Link to Source Code:} We use skfeature \cite{li2017feature}'s implementation of Laplacian Score feature selection algorithm from \url{https://github.com/jundongl/scikit-feature/blob/master/skfeature/example/test_lap_score.py}
        \item \textbf{Setting parameters:} We utilize the default setting in the skfeature's provided example. The number of clusters is set according to the number of distinct ground truth labels in each dataset. The number of neighbors is set to $5$ as suggested in the original paper.
        \end{itemize}
\end{itemize}

\end{document}